\RequirePackage{fix-cm}
\documentclass[smallextended]{svjour3}       
\smartqed  
\usepackage{graphicx}
\usepackage{graphicx}
\usepackage{natbib}
\usepackage{hyperref}       
\usepackage{url}            
\usepackage{amsfonts}       
\usepackage{tabularx}
\usepackage{amsmath}

\usepackage{amsthm}
\usepackage{blindtext}
\usepackage{enumitem}
\usepackage{multirow}
\usepackage{graphicx}
\usepackage{algorithm}
\usepackage{algpseudocode}
\usepackage{amsfonts}
\usepackage{xcolor}
\usepackage{subcaption}
\usepackage{booktabs}       
\usepackage{amsfonts}       
\usepackage{nicefrac}       
\usepackage{microtype}      
\usepackage{adjustbox}
\usepackage{tabularx}
\usepackage{array}
\usepackage{amsmath}
\usepackage{blindtext}
\usepackage{enumitem}
\usepackage{algorithm}
\usepackage{algpseudocode}
\usepackage{amsfonts}
\usepackage{comment}
\theoremstyle{definition}
\newtheorem{assumption}{Assumption}

\begin{document}

\title{Robust High Dimensional Expectation Maximization Algorithm via Trimmed Hard Thresholding 
}

\titlerunning{Robust High Dimensional EM Algorithm}        

\author{Di Wang  \and
        Xiangyu Guo  \footnote{The first two authors contributed equally.}   \and 
        Shi Li \and 
        Jinhui Xu 
}


\institute{Di Wang \at
                   Department of Computer Science and Engineering\\
                 Buffalo, NY, USA 14260\\
             State University of New York at Buffalo\\
              \email{dwang45@bufffalo.edu}           
           \and
           Xiangyu Guo \at
              Department of Computer Science and Engineering \\
             State University of New York at Buffalo, Buffalo\\
              NY, USA 14260 \\
             xiangyug@buffalo.edu
            \and 
            Shi Li \at 
               Department of Computer Science and Engineering \\
             State University of New York at Buffalo, Buffalo\\
              NY, USA 14260\\
                      \email{shil@bufffalo.edu}   
             \and 
             Jinhui Xu \at 
                Department of Computer Science and Engineering\\
                 Buffalo, NY, USA 14260\\
             State University of New York at Buffalo\\
               \email{jinhui@bufffalo.edu}   \\
             Corresponding author
}

\date{Received: date / Accepted: date}

\maketitle
\begin{abstract}
In this paper, we study the problem of estimating latent variable models with arbitrarily corrupted samples in high dimensional space ({\em i.e.,} $d\gg n$) where the underlying parameter is assumed to be sparse. Specifically, we propose a method called Trimmed (Gradient) Expectation Maximization which adds a trimming gradients step and a hard thresholding step to the  Expectation step (E-step) and  the Maximization step (M-step), respectively. We show that 
under some mild assumptions and with an appropriate initialization, 
the algorithm is corruption-proofing and converges to the (near) optimal statistical rate geometrically when the fraction of the corrupted samples  $\epsilon$ is bounded by $ \tilde{O}(\frac{1}{\sqrt{n}})$. Moreover, we apply our general framework to three canonical models: mixture of Gaussians, mixture of regressions and linear regression with missing covariates.  Our theory is supported by thorough numerical results. 

\end{abstract}

\section{Introduction}

{As one of the most popular techniques for estimating the maximum likelihood 
of mixture models or incomplete data problems, Expectation Maximization (EM) algorithm has been  widely applied to many areas such as genomics  \citep{laird2010algorithm}, finance \citep{faria2013financial}, and crowdsourcing \citep{dawid1979maximum}.  Although EM algorithm is well-known to converge to an empirically good local estimator \citep{wu1983convergence}, finite sample statistical guarantees for its performance have not been established until recent studies \citep{balakrishnan2017statistical}\citep{zhu2017high},\citep{wang2015high},\citep{yi2015regularized}. Specifically, the first local convergence theory and finite sample statistical rate of convergence for the classical EM and its gradient ascent variant (gradient EM) were established in \citep{balakrishnan2017statistical}. Later, \citep{wang2015high} extended the classical EM and gradient EM algorithms to the high dimensional sparse setting, and the key idea in their methods is an additional truncation step after the M-step, which can exploit the intrinsic sparse structure of the high dimensional latent variable models. Later on, \citep{yi2015regularized} also studied the high dimensional sparse EM algorithm and proposed a method which uses a regularized M-estimator in the M-step. Recently, \citep{zhu2017high} considered the computational issue of the previous methods of the problem in  high dimensional sparse case. They proposed a method called VRSGEM (Variance Reduced Stochastic Gradient EM) which combines the idea of SVRG (Stochastic Variance Reduced Gradient) \citep{johnson2013accelerating} and the high dimensional gradient EM algorithm. Their method has less gradient complexity while also can achieve almost the same statistical estimation errors as the previous ones.  }

{Although the above methods could achieve  (near) optimal minimax rate for some statistical models such as Gaussian mixture model, mixture of regressions and linear regression with missing covariates (see Preliminaries section for details),  all of these results need to assume that the data samples have no corruptions and also should satisfy some statistical assumptions, such as sub-Gaussian. This means that some arbitrary corruptions among the data samples may cause the dataset violate these statistical assumptions which are required for convergence of the above methods, or they will even make the above methods achieve unacceptable statistical estimation errors (see Figure \ref{fig:method_non} for experimental studies). Thus,  the classical EM algorithm and its variants are sensitive to these corruptions. Although statistical estimation with arbitrary corruptions has long been a focus in robust statistics \citep{huber2011robust}, it is still unknown that \textbf{whether there exist some variant of  (gradient) EM algorithm which is robust to arbitrary corruptions while also has finite sample statistical guarantees as in the non-corrupted case}.}

To address the aforementioned issue, in this paper, we study the problem of statistical estimation of latent variable models with arbitrarily corrupted samples in high dimensional space\footnote{Since high dimensional sparse case is much more harder than the low dimension case, our algorithm can be easily extended to the low dimension case by using the results in \citep{balakrishnan2017statistical}. Due to the space limit, we omit it in the paper.} ({\em i.e.,} $d\gg n$) where the underlying parameter is assumed to be sparse. Specifically, 
we propose  a new algorithm called Trimmed (Gradient) Expectation Maximization, which attaches a trimming gradient and hard thresholding step to the E-step and M-step in each iteration, respectively. We show that under certain conditions, our algorithm is robust against corruption and converges with a statistical estimation error which is (near) statistically optimal. 
Below is a summary of our main contributions.
\begin{enumerate}
    \item We show that, given an appropriate initialization 
    $\beta^{\text{init}}$, {\em i.e.,} $\|\beta^{\text{init}}-\beta^*\|\leq \kappa \|\beta^*\|_2$ for some constant $\kappa\in (0,1)$, if the model satisfies some additional assumptions, the iterative solution sequence $\beta^{t}$ of our algorithm satisfies $\|\beta^t-\beta^*\|_2\leq \tilde{O}\big( c_1 \rho^{t}+\sqrt{s^*}c_2(\epsilon\log(nd)+\sqrt{\frac{\log d}{n}})\big)$ with high probability, where $\rho\in (0,1)$, $c_1, c_2$ are some constants 
    dependent on the model, $\epsilon$ is the fraction of the perturbed samples, and $s^*$ is the sparsity parameter of 
    the underlying parameter 
    $\beta^*$. Particularly, when $c_2$ is a constant and $\epsilon\leq O(\frac{1}{\sqrt{n}\log(nd)})$, the above estimation error geometrically converges to $O(\sqrt{\frac{s^*\log d}{n}})$, which is statistically optimal. This means that our algorithm is corruption-proofing for a certain level of corruption 
    that is only dependent on the sample size, which is quite useful in the high dimensional setting. 
    \item We implement our algorithm on three canonical models: mixture of Gaussians, mixture of regressions and linear regression with missing covariates. 
    Experimental results on these models support our theoretical analysis.
\end{enumerate}
Some background, lemmas and all the proofs are included in the Appendix.

\section{Related Work}

{There are mainly two perspectives on the study of EM algorithm. The first one focuses on its statistical guarantees \citep{balakrishnan2017statistical,zhu2017high,wang2015high,yi2015regularized}. However, there are many differences compared with our results. Firstly, as we mentioned above, although in this paper we study the same statistical setting as these previous work, our method is corruption-proofing while the performance of their algorithms is heavily affected by outliers. Secondly, in our paper we use a robust version of the gradient instead of the original gradient, this make the proof of our theoretical result different with the above previous papers. Another direction focus on the practical performance, and there are many robust variants of the EM algorithm such as \citep{aitkin1980mixture,yang2012robust}. However, we note that these methods are incomparable with ours. Firstly, in this paper we mainly focus on statistical setting and the statistical guarantees while there is no any theoretical guarantees of these methods. Secondly, previous methods can only be used in the low dimension case while we focus on the high dimensional sparse case. Thus, to our best knowledge, there is no previous work on the variants of the EM algorithm that is both
robust  to some corruptions 
and also has statistical guarantees. Thus, in the following we will only compare with some other methods that are close to ours.}


\citep{diakonikolas2016robust,diakonikolas2018list,diakonikolas2017statistical,chen2013robust} studied the problem of robustly estimating the mixture of distributions. However, some of them 
are not computationally practical as they rely on the rather time-consuming ellipsoid method. 
Moreover, these methods in general cannot be extended to the distributed or Byzantine setting \citep{chen2017distributed}, while ours can be easily extended to such scenarios.  

\citep{du2017computationally,balakrishnan2017computationally,li2017robust,suggala2019adaptive,dalalyan2019outlier,thompson2018restricted} studied the robust high dimensional sparse estimation problem for some specified tasks, such as GLM, linear regression, mean and covariance matrix estimation. However, none of them 
considered estimating the latent variable models and thus is quite different from ours.

Recently, several  
robust methods have been proposed based on 
(stochastic) gradient descent, such as \citep{alistarh2018byzantine,chen2017distributed,yin2018byzantine,prasad2018robust,holland2018robust}. However, none of them studies the latent variable models and 
all of them consider only the low dimensional case.

 We have to note that the most closed work to ours is given by \citep{liu2019high}. Specifically,  \citep{liu2019high} recently investigated the robust high dimensional sparse M-estimation problem (such as linear regression and logistic regression) by combining hard thresholding with trimming steps. However, their results are incomparable with ours. Particularly, their method can only be used in the M-estimation, and they only consider the case where the loss function is convex while ours focuses on the latent variable model and the EM algorithm, and the loss function ($Q$-function) is non-convex. Thus, we cannot use their proofs directly to get our theoretical results.

\section{Preliminaries}\label{prelin}

Let $Y$ and $Z$ be two random variables taking values in the sample spaces $\mathcal{Y}$ and $\mathcal{Z}$, respectively. Suppose that the pair $(Y, Z)$ has a joint density function $f_{\beta^*}$ that belongs to some parameterized family $\{f_{\beta^*}|\beta^* \in \Omega\}$. Rather than considering the whole pair of $(Y, Z)$, we observe only component $Y$. Thus,  component $Z$ can be viewed as the missing or latent structure. We assume that the term $h_{\beta}(y)$ is the marginal distribution over the latent variable $Z$, {\em i.e.,} $h_\beta(y)=\int_{\mathcal{Z}} f_\beta(y, z) dz.$ Let $k_{\beta}(z|y)$ be the density of $Z$ conditional on the observed variable $Y=y$, that is, $k_\beta(z|y)=\frac{f_\beta(y, z)}{h_\beta(y)}.$

Given $n$ observations $y_1, y_2, \cdots, y_n$ of $Y$, the EM algorithm is to maximize the log-likelihood $\max_{\beta\in \Omega}\ell_n(\beta) = \sum_{i=1}^n\log h_\beta(y_i).$
Due to the unobserved latent variable $Z$, it is often difficult to directly evaluate $\ell_n(\beta)$. Thus, we 
consider the lower bound of $\ell_n(\beta)$ . By Jensen's inequality, we have 
\begin{align}
   & \frac{1}{n}[\ell_n(\beta)-\ell_n(\beta')]
   \geq \frac{1}{n}\sum_{i=1}^n\int_{\mathcal{Z}}k_{\beta'}(z|y_i)\log f_\beta(y_i, z)dz \nonumber \\
   &- \frac{1}{n}\sum_{i=1}^n\int_{\mathcal{Z}}k_{\beta'}(z|y_i)\log {f_{\beta'}(y_i, z)}dz.   \label{eq:1}
\end{align}
Let $Q_n(\beta; \beta')=\frac{1}{n}\sum_{i=1}^n q_i(\beta;\beta') $, where 
\begin{equation}\label{eq:2}
   q_{i}(\beta; \beta')=\int_{\mathcal{Z}}k_{\beta'}(z|y_i)\log f_\beta(y_i, z)dz. 
\end{equation}
Also, it is convenient to let   $Q(\beta; \beta')$ denote the expectation of $Q_n(\beta; \beta')$ w.r.t $\{y_i\}_{i=1}^n$, that is,
\begin{equation}\label{eq:3}
    Q(\beta; \beta')= \mathbb{E}_{y\sim h_{\beta^*}}\int_{\mathcal{Z}}k_{\beta'}(z|y)\log f_\beta(y, z)dz.
\end{equation}
We can see that the second term on the right hand side of (\ref{eq:1}) is not dependent on $\beta$. Thus, given some fixed $\beta'$, we can maximize the lower bound function $Q_n(\beta; \beta')$ over $\beta$ to obtain sufficiently large $\ell_n(\beta)-\ell_n (\beta')$. Thus, in the $t$-th iteration of the standard EM algorithm, we can evaluate $Q_n(\cdot; \beta^t)$ at the E-step and then perform the operation of $\max_{\beta\in \Omega}Q_n(\beta; \beta^t)$ at the M-step. See \citep{mclachlan2007algorithm} for more details.

In addition to the exact maximization implementation of the M-step, we add a gradient ascent implementation of the M-step, which performs an approximate maximization via a gradient descent step. 
\vspace{0.1in}

\noindent \textbf{Gradient EM Procedure \citep{balakrishnan2017statistical}} When $Q_n(\cdot; \beta^t)$ is differentiable, the update of $\beta^t$ to $\beta^{t+1}$ consists of the following two steps. 
\begin{itemize}
    \item E-step: Evaluate the functions in (\ref{eq:2}) to compute $Q_n(\cdot; \beta^t)$.
    \item M-step: Update $\beta^{t+1}=\beta^t+\eta \nabla Q_n(\beta^t; \beta^t)$, where $\nabla$ is the derivative of $Q_n$ w.r.t the first component and $\eta$ is the step size. 
\end{itemize}
Next, we give some examples that use the gradient EM algorithm. Note that they are the typical examples for studying the statistical property of EM algorithm \citep{wang2015high,balakrishnan2017statistical,yi2015regularized,zhu2017high}. 

\vspace{0.1in}
\noindent  \textbf{Gaussian Mixture Model} Let $y_1, \cdots, y_n$ be $n$ i.i.d. samples from $Y\in \mathbb{R}^d$ with 
\begin{equation}\label{eq:4}
    Y = Z\cdot \beta^*+V,
\end{equation}
where $Z$ is a Rademacher random variable ({\em i.e.,} $\mathbb{P}(Z=+1)= \mathbb{P}(Z=-1)=\frac{1}{2}$),  and $V\sim \mathcal{N}(0, \sigma^2 I_d)$ is independent of $Z$ for some known standard deviation $\sigma$. In our high dimensional setting, we assume that  $\|\beta^*\|_0=s^*$ is sparse. \footnote{For a vector $v\in \mathbb{R}^d$, $\|v\|_0$ represents the number of entries in $v$ that are non-zero.}

For Gaussian Mixture Model, we have 
\begin{equation}\label{eq:5}
    \nabla q_i(\beta;\beta)=[2w_\beta(y_i)-1]\cdot y_i-\beta,
\end{equation}
where $w_\beta(y)=\frac{1}{1+\exp(-\langle \beta, y\rangle/\sigma^2)}$.

\vspace{0.1in}

\noindent \textbf{Mixture of (Linear) Regressions Model} Let $n$ samples $(x_1, y_1)$, $(x_2, y_2)$, $\cdots, (x_n, y_n)$ i.i.d.. sampled from $Y\in \mathbb{R}$ and $X\in \mathbb{R}^d$  with 
\begin{equation}\label{eq:6}
    Y= Z\langle \beta^*, X \rangle +V,  
\end{equation}
where $X\sim \mathcal{N}(0, I_d)$, $V\sim \mathcal{N}(0, \sigma^2)$\footnote{$\langle \cdot, \cdot \rangle$ represents the inner product of two vectors.}, $Z$ is a Rademacher random variable, and $X, V, Z$ are independent. In the high dimensional case, we assume that  $\|\beta^*\|_0=s^*$ is sparse.

In this case, we have 
\begin{equation}\label{eq:7}
    \nabla q_i(\beta;\beta)=(2w_\beta(x_i, y_i)-1) \cdot y_i \cdot x_i-x_i x_i^T\cdot \beta,
\end{equation}
where $w_\beta(x_i, y_i)=\frac{1}{1+\exp(-y\langle \beta ,x \rangle/\sigma^2)}$.

\vspace{0.1in}
\noindent  \textbf{Linear Regression with Missing Covariates} We assume that $Y\in \mathbb{R}$ and $X\in \mathbb{R}^d$ satisfy 
\begin{equation}\label{eq:8}
    Y= \langle X, \beta^* \rangle +V,
\end{equation}
where $X\sim \mathcal{N}(0, I_d)$ and $V\sim \mathcal{N}(0, \sigma^2)$ are independent. In our high dimensional setting, we assume that $\|\beta^*\|_0=s^*$ is sparse. Let $x_1, x_2, \cdots, x_n$ be $n$ observations of $X$ with each coordinate of $x_i$ missing (unobserved) independently with probability $p_m\in[0,1)$. 

In this case, we have \begin{equation}\label{eq:9}
    \nabla q_i(\beta; \beta)= y_i\cdot m_\beta(x_i^{\text{obs}},y_i)-K_\beta(x_i^{\text{obs}}, y_i)\beta,
\end{equation}
where the functions $m_\beta(x_i^{\text{obs}},y_i)\in \mathbb{R}^d$ and $K_\beta(x_i^{\text{obs}}, y_i)\in \mathbb{R}^{d\times d}$ are defined as:   
\begin{equation}\label{eq:10}
    m_\beta(x_i^{\text{obs}},y_i)= z_i \odot x_i+\frac{y_i-\langle \beta, z_i\odot x_i\rangle }{\sigma^2+\|(1-z_i)\odot \beta\|_2^2}(1-z_i)\odot \beta 
\end{equation}
and 
\begin{multline}\label{eq:11}
    K_\beta(x_i^{\text{obs}}, y_i)=\text{diag}(1-z_i)+  m_\beta(x_i^{\text{obs}},y_i)\cdot [  m_\beta(x_i^{\text{obs}},y_i)]^T \\
    -[(1-z_i)\odot   m_\beta(x_i^{\text{obs}},y_i)]\cdot [(1-z_i)\odot   m_\beta(x_i^{\text{obs}},y_i)]^T,
\end{multline}
where vector $z_i \in \mathbb{R}^d$ is defined as $z_{i,j}=1$ if $x_{i,j}$ is observed and $z_{i,j}=0$ is $x_{i,j}$ is missing, and $\odot$ denotes the Hadamard product of matrices.

Next, we provide several definitions on the required properties of functions $Q_n(\cdot; \cdot)$ and $Q(\cdot; \cdot)$. Note that some of them have been used in the previous studies on EM  \citep{balakrishnan2017statistical,wang2015high,zhu2017high}.

\begin{definition}
 Function $Q(\cdot; \beta^*)$ is self-consistent if  
    $\beta^*=\arg\max_{\beta\in \Omega}Q(\beta; \beta^*).$
That is, $\beta^*$ maximizes the lower bound of the log likelihood function.
\end{definition}

\begin{definition}[Lipschitz-Gradient-2($\gamma, \mathcal{B}$)]\label{def:1}
$Q(\cdot; \cdot)$ is called Lipschitz-Gradient-2($\gamma, \mathcal{B}$), if for the underlying parameter $\beta^*$ and any $\beta\in \mathcal{B}$ for some set $\mathcal{B}$, the following holds 
\begin{equation}\label{eq:12}
    \|\nabla Q(\beta; \beta^*)-\nabla Q(\beta; \beta)\|_2\leq \gamma \|\beta-\beta^*\|_2.
\end{equation}
\end{definition}

We note that there are some differences between the definition of Lipschitz-Gradient-2 and the Lipschitz continuity condition in the convex optimization literature \citep{nesterov2013introductory}. Firstly, in (\ref{eq:12}), the gradient is w.r.t the second component, while the Lipschitz continuity is w.r.t the first component. Secondly, the property holds only for fixed $\beta^*$ and any $\beta$, while the Lipschitz continuity is for all $\beta, \beta'\in \mathcal{B}$.

\begin{definition}[$\mu$-smooth]
$Q(\cdot; \beta^*)$ is $\mu$-smooth, that is if
for any $\beta, \beta'\in \mathcal{B}$, 
$ Q(\beta;\beta^*)\geq Q(\beta'; \beta^*)+(\beta-\beta')^T\nabla Q(\beta';\beta^*)-\frac{\mu}{2}\|\beta'-\beta\|_2^2.$
\end{definition}

\begin{definition}[$\upsilon$-strongly concave] \label{def:3} $Q(\cdot; \beta^*)$ is $\upsilon$-strongly concave, that is if 
for any $\beta, \beta'\in \mathcal{B}$,  $Q(\beta;\beta^*)\leq  Q(\beta'; \beta^*)+(\beta-\beta')^T\nabla Q(\beta';\beta^*)-\frac{\upsilon}{2}\|\beta'-\beta\|_2^2.$
\end{definition}

Next, we assume that each coordinate of $\nabla q(\beta; \beta)$ in (\ref{eq:2}) is sub-exponential for every $\beta\in \mathcal{B}$,  where $\nabla$ is the derivative of $q$ w.r.t the first component.
\begin{definition}[$\xi$-sub-exponential] \label{def:4}
A random variable $X$ with mean $\mathbb{E}(X)$ is $\xi$-sub-exponential for $\xi>0$ if 
for all $|t|<\frac{1}{\xi}$,
    $\mathbb{E}\{\exp(t[X-\mathbb{E}(X)])\}\leq \exp(\frac{\xi^2t^2}{2}). $
\end{definition}

\begin{assumption}\label{ass:1}
We assume that $Q(\cdot; \cdot)$ in (\ref{eq:3}) is self-consistent, Lipschitz-Gradient-2($\gamma, \mathcal{B}$), $\mu$-smooth and $\upsilon$-strongly convex for some $\mathcal{B}$. Moreover, we assume that for any fixed $\beta\in\mathcal{B}$ with $\|\beta\|_0\leq s$ (where the value of $s$ will be specified later) and $\forall j\in [d]$,  the $j$-th coordinate of  $\nabla q(\beta; \beta)$ ({\em i.e.,} $[\nabla q(\beta; \beta)]_j$) is $\xi$-sub-exponential and for each $i\in [n]$, $[\nabla q_i(\beta, \beta)]_j$ is independent with others. 
\end{assumption}

We note that the sub-exponential assumption on each coordinate is stronger than the assumption of Statistical-Error in \citep{wang2015high,balakrishnan2017statistical}. However, since 
the model considered in this paper could have arbitrarily corrupted samples, we will see later that this assumption is necessary. 

Finally, we give the definition of the corruption model studied 
in the paper. 
\begin{definition}[$\epsilon$-corrupted samples ]
Let $\{y_1, y_2, \cdots, y_n\}$ be $n$ i.i.d. observations with distribution $P$. We say that a collection of samples $\{z_1, z_2, \cdots, z_n\}$ is $\epsilon$-corrupted if an adversary chooses an arbitrary $\epsilon$-fraction of the samples in $\{y_i\}_{i=1}^n$ and modifies them with arbitrary values. 
\end{definition}
We note that this is a quite common model in robust estimation or robust statistics. Equivalently, it means that there are $\epsilon$-fraction of samples in the dataset are outliers (or they are corrupted arbitrarily). 
\section{Trimmed Expectation Maximization Algorithm}

To obtain a robust estimator for the high dimensional model with $\epsilon$-corrupted samples, we propose a trimmed EM algorithm, which is based on the gradient EM algorithm. See Algorithm \ref{alg:1} for details. 

Note that compared with the previous gradient EM algorithm, Trimmed EM algorithm has two additional steps in each iteration, {\em i.e.,} the trimming gradient and hard thresholding step. For the trimming gradient step 4 in Algorithm \ref{alg:1}, we use the dimensional $\alpha$-trimmed estimator ({\em i.e.,} $\text{D-Trim}_{\alpha}$) on the gradients $\{\nabla q_i(\beta^t; \beta^t)\}_{i=1}^n$. We note that while this operator has also been studied in \citep{liu2019high,yin2018byzantine} for the M-estimators, we use it for the EM algorithm. Here is the definition of the function $\text{D-Trim}_{\alpha}(\cdot)$.

\begin{definition}[Dimensional $\alpha$-trimmed estimator] Given a set of $\epsilon$-corrupted samples in the form of $d$-dimensional vectors $\{z_i\}_{i=1}^n$, the D-Trim operator $\text{D-Trim}_{\alpha}(\{z_i\}_{i=1}^n)\in \mathbb{R}^d$ performs as follows. For each dimension $j\in [d]$, it first removes the largest and the smallest $\alpha$ fraction of elements in the $j$-th coordinate of $\{z_i\}_{i=1}^n$, {\em i.e.,} $\{z_{i,j}\}_{i=1}^n$, and then calculates the mean of the remaining terms, where  $\alpha=c_0\epsilon$ and $\alpha\leq \frac{1}{2}-c_1$ for some constant $c_0\geq 1$ and 
a small constant $c_1$. 
\end{definition}

The rationale behind the use of 
the dimensional trimmed estimator is that due to the existence of 
$\epsilon$ fraction of corrupted samples,
directly calculating the the mean of the gradient could introduce a large error to the population gradient $\nabla Q(\beta^t; \beta^t)$ in (\ref{eq:3}). 
Also, it can be shown that if each coordinate of $\nabla q_i(\beta^t; \beta^t)$ is sub-exponential,  it will be robust against the $\epsilon$-corruption for some small $\epsilon$. This motivates us to use the dimensional trimmed operation.  

\begin{algorithm}[t]
		\caption{Trimmed (Gradient) Expectation Maximization}
	$\mathbf{Input}$:  T is the iteration number, $\beta^{\text{init}}$ is the initial parameter, $\eta$ is the flexed step-size and $s$ is the sparsity parameter to be specified later. $\{z_i\}_{i=1}^n$ are the $\epsilon$ corrupted samples of $\{y_i\}_{i=1}^n$.
	\label{alg:1}
	\begin{algorithmic}[1]
      \State Let $\hat{\mathcal{S}}^{\text{init}}=\text{supp}(\beta^{\text{init}}, s)$, $\beta^{0}=\text{trunc}(\beta^{\text{init}}, \hat{\mathcal{S}}^{\text{init}})$. 
      \For {$t=0, 1, \cdots, T-1$}
      \State {\bf E-step:} Evaluate $\{\nabla q_i(\beta^t; \beta^t)\}_{i=1}^n$  for $\{z_i\}_{i=1}^n$.
      \State {\bf Trimming step:} Use a dimensional $\alpha$-trimmed gradient estimator to get a vector $\nabla\tilde{Q}_n(\beta^t; \beta^t)=\text{D-Trim}_{\alpha}(\{\nabla q_i(\beta^t; \beta^t)\}_{i=1}^n)$.
      \State {\bf M-step:} Update $\beta^{t+0.5}=\beta^t+\eta\nabla\tilde{Q}_n(\beta^t; \beta^t).$
      \State {\bf Thresholding step:} Let $\hat{\mathcal{S}}^{t+0.5}=\text{supp}(\beta^{t+0.5}, s)$ and $\beta^{t+1}=\text{trunc}(\beta^{t+0.5}, \hat{\mathcal{S}}^{t+0.5})$
      \EndFor
      \State Return $\beta^T$. 
	\end{algorithmic}
\end{algorithm}

To ensure the sparsity of our estimator, after getting $\beta^{t+0.5}$, we need to use the hard thresholding operation \citep{blumensath2009iterative}. More specifically, we first find the set  $\hat{\mathcal{S}}^{t+0.5}\subseteq [d]$ of 
indices $j$ corresponding to the top $s $ largest $|\beta^{t+0.5}_j|$ (we denote $\hat{\mathcal{S}}^{t+0.5}=\text{supp}(\beta^{t+0.5}, s)$\footnote{In general, given a vector $v\in \mathbb{R}^d$ and an integer $s$, function $\text{supp}(v,s)$ returns a set of $s$ number ofis indices corresponding to the top $s$ largest value among $\{|v_j|, j\in [d]\}$. }), and make the value of the remaining entries
$\beta^{t+0.5}_j$ for $j\in [d]\backslash \hat{\mathcal{S}}^{t+0.5}$ be $0$ (we denote $\beta^{t+1}=\text{trunc}(\beta^{t+0.5}, \hat{\mathcal{S}}^{t+0.5})$\footnote{In general, given a vector $v\in\mathbb{R}^d$ and a set of indices $\mathcal{S}\subseteq [d]$, function $\text{trunc}(v, \mathcal{S})\in \mathbb{R}^d$, where $[\text{trunc}(v, \mathcal{S})]_j=v_j$ if $j\in \mathcal{S}$ and $[\text{trunc}(v, \mathcal{S})]_j=0$ otherwise.}). The sparsity level $s$ controls the sparsity of the estimator and  the estimation error. 

The following main theorem shows that under Assumption \ref{ass:1} and with some proper initial vector $\beta^{\text{init}}$, the estimator $\beta^T$ converges to the underlying $\beta^*$ at a geometric rate with high probability. 

\begin{theorem}\label{thm:1}
Let $\mathcal{B}=\{\beta: \|\beta-\beta^*\|_2\leq R\}$ be a set with  $R=k\|\beta^*\|_2$ for some $k\in (0,1)$. Assume that Assumption \ref{ass:1} holds for parameters  $\mathcal{B}, \gamma, \mu, \upsilon, \xi$ satisfying the condition of  $1-2\frac{\upsilon-\gamma}{\upsilon+\mu}\in (0,1)$ and the sparsity parameter $s$ is chosen to be 
\begin{align}\label{eq:13}
    s=\lceil C\max\{\frac{16}{\{1/[1-2(\upsilon-\gamma)/(\upsilon+\mu)]-1\}^2}, \frac{4(1+k)^2}{(1-k)^2}\} s^* \rceil,
\end{align}
where $C$ is some absolute constant. 
Also, assume that  $\|\beta^{\text{init}}-\beta^*\|_2\leq \frac{R}{2}$ and  there exist some absolute constants $C_1$ and $C_2$ satisfying the condition of 
\begin{multline}\label{eq:14}
   \frac{1}{\upsilon+\mu} C_2 (\sqrt{s}+\frac{C_1\sqrt{s^*}}{\sqrt{1-k}})\xi (\epsilon\log(nd)+\sqrt{\frac{\log d}{n}})\\ 
    \leq \min\big\{\big(1-\sqrt{1-\frac{2(\upsilon-\gamma)}{\upsilon+\mu}}\big)^2 R, \frac{(1-k)^2}{2(1+k)}\|\beta^*\|_2\big\}.
\end{multline}
Then, if taking $\eta=\frac{2}{\upsilon+\mu}$ in Algorithm \ref{alg:1}, the following holds for  $t=1, \dots, T$ with probability at least $1-Td^{-3}$ 
\begin{multline}\label{eq:15}
   \|\beta^t-\beta^*\|_2\leq \underbrace{(1-2\frac{\upsilon-\gamma}{\upsilon+\mu})^{\frac{t}{2}} R}_{\text{Optimization Error}} +  \underbrace{\frac{2C_2\xi(\epsilon\log(nd)+\sqrt{\frac{\log d}{n}})}{\upsilon+\mu}
  \frac{\sqrt{s}+\frac{C_1}{\sqrt{1-k}}\sqrt{s^*}}{1-\sqrt{1-2\frac{\upsilon-\gamma}{\upsilon+\mu}}}}_{\text{Statistical and  Corruption Error}}.
\end{multline}
\end{theorem}

In the above theorem, assumption (\ref{eq:13}) indicates that the sparsity level $s$ in Algorithm \ref{alg:1} should be sufficiently large but still in  
the same order as the underlying sparsity $s^*$. Although $s$ seems quite complex, in the experiments, we can see that it is suffcient to set $s=s^*$. Assumption (\ref{eq:14}) suggests that in order to ensure an upper bound in the hard thresholding step, we need $\sqrt{s^*}\xi(\epsilon\log(nd)+\sqrt{\frac{\log d}{n}})\leq O(\|\beta^*\|_2)$, which means that $n$ should be sufficiently large and the fraction of corruption $\epsilon$ cannot be too large. In the error bound of (\ref{eq:15}), there are three types of errors. The first one is caused by optimization, which decreases to zero at a geometric rate of convergence. The second one is the term related to $\epsilon$ ({\em i.e.,} $O(\xi\sqrt{s^*}\epsilon \log(nd))$), which is caused by estimating the population gradient via the trimming step due the $\epsilon$-corrupted samples. In the special case of 
no corrupted samples ({\em i.e.,} $\epsilon=0$), the bound will be zero. The third one is the term $O(\xi\sqrt{\frac{s^*\log d}{n}})$, which corresponds to the statistical error. It is independent of both $\epsilon$ and $t$ and only dependent on the model itself. Even though Theorem 1 requires that the initial estimator be close enough to the optimal one, our
experiments show that the algorithm actually performs quite well for any random initialization. 

From Theorem \ref{thm:1}, we can also see that when the fraction of corruption $\epsilon$ is sufficiently small such that $\epsilon\leq O(\frac{1}{\sqrt{n\log (nd)}})$ and the iteration number is sufficiently large, the error bound in (\ref{eq:15}) becomes $O(\xi\sqrt{\frac{s^*\log d }{n}})$, which is the same as the optimal rate of estimating a high dimensional sparse vector when $\xi$ is some constant. This means that our method has the same rate as the non-corrupted ones in \citep{wang2015high}. This rate of corruption also has been appeared in the corrupted sparse linear regression \citep{dalalyan2019outlier,liu2019high}. Also, we can see that when $\alpha=0$, our algorithm will be reduced to the high dimensional gradient EM algorithm in \citep{wang2015high}.

\section{Implications for Some Specific Models}\label{sec:implications}

In this section, we apply our framework ({\em i.e.,} Algorithm \ref{alg:1}) to the models  mentioned in Section \ref{prelin}.
To obtain results for these models, we only need to find the corresponding $\mathcal{B}, \gamma, k, R, \upsilon, \mu, \xi$ to ensure that Assumption \ref{ass:1} and assumptions in Theorem \ref{thm:1} hold. 

\subsection{Corrupted Gaussian Mixture Model}

The following lemma, which was given in  \citep{balakrishnan2017statistical}, ensures the properties of Lipschitz-Gradient-2($\gamma, \mathcal{B}$), smoothness and strongly concave for  model (\ref{eq:4}). It is easy to show that the model is self-consistent \citep{yi2015regularized}.

\begin{lemma}[\citep{balakrishnan2017statistical,yi2015regularized}]\label{lemma1}
If 
$\frac{\|\beta^*\|_2}{\sigma}\geq r$, where $r$ is a sufficiently large constant denoting the minimum signal-to-noise ratio (SNR), then there exists an absolute constant $C>0$ such that the properties of self-consistent, Lipschitz-Gradient-2($\gamma, \mathcal{B})$, $\mu$-smoothness and $\upsilon$-strongly concave hold for function $Q(\cdot; \cdot)$  with
    $\gamma=\exp(-Cr^2), \mu=\upsilon=1, R=k\|\beta^*\|_2, k=\frac{1}{4}, \text{ and } \mathcal{B}=\{\beta:\|\beta-\beta^*\|_2\leq R\}.$
\end{lemma}
\begin{lemma}\label{lemma2}
With the same notations as 
in Lemma \ref{lemma1}, for each $\beta\in \mathcal{B}$ with $\|\beta\|_0\leq s$, the $j$-th coordinate of $\nabla q_i(\beta; \beta)$ is $\xi$-sub-exponential with 
\begin{equation}\label{eq:17}
    \xi = C_1\sqrt{\|\beta^*\|^2_{\infty}+\sigma^2},
\end{equation}
where $C_1$ is some absolute constant.  Also, each $[\nabla q_i(\beta;\beta)]_j$, where $i\in [n]$, is independent of others for any fixed $j\in[d]$.  
\end{lemma}
\begin{theorem}\label{thm:2}
In an $\epsilon$-corrupted high dimensional Gaussian Mixture Model with $\epsilon$ satisfying the condition of  
\begin{equation}\label{eq:18}
    \sqrt{(\|\beta^*\|^2_{\infty}+\sigma^2)}\sqrt{s^*}(\epsilon \log(nd)+\sqrt{\frac{\log d}{n}})\leq O(\|\beta\|_2^*),
\end{equation} 
if 
$\frac{\|\beta^*\|_2}{\sigma}\geq r$ for some  sufficiently large constant $r$ denoting the minimum SNR and the initial estimator $\beta^{\text{init}}$ satisfies the inequality  of  $\|\beta^{\text{init}}-\beta^*\|_2\leq \frac{1}{8}\|\beta^*\|_2,$
then the  output $\beta^T$ of  Algorithm \ref{alg:1} after choosing
$s=O(s^*)$ and $\eta=O(1)$ 
satisfies the following 
with probability at least $1-Td^{-3}$
\begin{multline}\label{eq:19}
   \|\beta^T-\beta^*\|_2\leq \exp(-C T r^2)\|\beta^*\|_2 \\
   +O\big( \sqrt{(\|\beta^*\|^2_{\infty}+\sigma^2)}\sqrt{s^*}(\epsilon \log(nd)+\sqrt{\frac{\log d}{n}})\big),
\end{multline}

where $C$ is some absolute constant. 
\end{theorem}
From Theorem \ref{thm:2}, we can see that when $\epsilon\leq \tilde{O}(\frac{1}{\sqrt{n}})$ and $T=O(\log \frac{n}{s^*\log d})$, the output achieves an estimation error of $O(\sqrt{\frac{s^*\log d}{n}})$, which matches the best-known error bound of the no-outlier case \citep{yi2015regularized,wang2015high}. Also, we assume that the SNR is large, which is reasonable since it has been shown that for  Gaussian Mixture Model with low SNR, the variance of noise makes it harder for the algorithm to converge \citep{ma2000asymptotic}.

\subsection{Corrupted Mixture of Regressions Model}

The following lemma, which was given in \citep{balakrishnan2017statistical,yi2015regularized}, shows the properties of  Lipschitz-Gradient-2($\gamma, \mathcal{B}$), smoothness and strongly concave for model (\ref{eq:6}). 

\begin{lemma}[\citep{balakrishnan2017statistical,yi2015regularized}]\label{lemma:3}
If $\frac{\|\beta^*\|_2}{\sigma}\geq r$, where $r$ is a sufficiently large constant denoting 
the required minimal signal-to-noise ratio (SNR), then function $Q(\cdot; \cdot)$ of the Mixture of Regressions Model has the properties  of self-consistent, Lipschitz-Gradient-2($\gamma, \mathcal{B})$, $\mu$-smoothness, and $\upsilon$-strongly with 
    $\gamma\in(0,\frac{1}{4}), \mu=\upsilon=1, \mathcal{B}=\{\beta: \|\beta-\beta^*\|_2\leq R\}, R=k\|\beta^*\|_2$, and $k=\frac{1}{32}. $
\end{lemma}
\begin{lemma}\label{lemma:4}
With the same 
notations as in Lemma \ref{lemma:3}, for each $\beta\in \mathcal{B}$ and $\|\beta\|_0=s$, the $j$-th coordinate of $\nabla q_i(\beta; \beta)$  is $\xi$-sub-exponential with 
\begin{equation}\label{eq:21}
    \xi = C\max\{\|\beta^*\|^2_2+\sigma^2, 1, \sqrt{s}\|\beta^*\|_2\},
\end{equation}
where $C>0$ is some absolute constant. Also, each $[\nabla q_i(\beta;\beta)]_j$, where $i\in [n]$, is independent of others for any fixed $j\in[d]$. 
\end{lemma}
\begin{theorem}\label{thm:3}
In an $\epsilon$-corrupted high dimensional Mixture of Regressions Model with $\epsilon$ satisfying the condition of 
\begin{multline}\label{eq:22}
    \max\{\|\beta^*\|_2+\sigma^2, 1, \sqrt{s^*}\|\beta^*\|_2\}\sqrt{s^*}
    (\epsilon \log(nd) +\sqrt{\frac{\log d}{n}})
    \leq O(\|\beta\|_2^*),
\end{multline}
if 
$\frac{\|\beta^*\|_2}{\sigma}\geq r$ for  
some 
sufficiently large constant $r$ 
denoting the minimum SNR and the initial estimator $\beta^{\text{init}}$ satisfies the inequality of  $\|\beta^{\text{init}}-\beta^*\|_2\leq \frac{1}{64}\|\beta^*\|_2,$
then the output  $\beta^T$ of Algorithm \ref{alg:1} after choosing 
$s=O(s^*)$ and $\eta=O(1)$ satisfies the following 
with probability at least $1-Td^{-3}$
\begin{align}\label{eq:23}
   \|\beta^T-\beta^*\|_2 &\leq \gamma^{\frac{T}{2}}\|\beta^*\|_2+
  O\big(  \max\{\|\beta^*\|_2+\sigma^2, 1, \sqrt{s^*}\|\beta^*\|_2\}\nonumber \\
  &
   \times\sqrt{s^*}(\epsilon \log(nd)+\sqrt{\frac{\log d}{n}})\big),
\end{align}
where $\gamma\in (0, \frac{1}{4})$ is a constant. 
\end{theorem}

Note that in the above theorem, when $\epsilon\leq \tilde{O}(\frac{1}{\sqrt{n}})$ and $T=O(\log \frac{\sqrt{n}}{\sqrt{\log d} s^*})$, the estimation error becomes $O(s^*\sqrt{\frac{\log d}{n}})$, which differs from the $O(\sqrt{\frac{s^*\log d}{n}})$ minimax lower bound by only a factor of $\sqrt{s^*}$. We leave it as an open problem for further improvement. Recently, \citep{chen2018convex} shows that in the no-outlier and low dimensional setting, an assumption of $SNR\geq \rho$ for some constant $\rho$ is necessary for achieving the optimal rate $\Theta(\sqrt{\frac{d}{n}})$. 
\subsection{Corrupted Linear Regression with Missing Covariates}

\begin{lemma}[\citep{balakrishnan2017statistical,yi2015regularized}]\label{lemma:5}
If $\frac{\|\beta^*\|_2}{\sigma}\leq r$  and 
    $p_m<\frac{1}{1+2b+2b^2}$,
where  $r$ is a constant denoting the required maximum signal-to-noise ratio (SNR) and   $b=r^2(1+k)^2$ for some constant $k\in (0,1)$, then  function $Q(\cdot; \cdot)$ of the linear regression with missing covariates has the properties of self-consistent, Lipschitz-Gradient-2($\gamma, \mathcal{B})$, $\mu$-smoothness and $\upsilon$-strongly with 
\begin{align}\label{eq:25}
    &\gamma =\frac{b+p_m(1+2b+2b^2)}{1+b}<1, \mu=\upsilon=1,\nonumber \\ &\mathcal{B}=\{\beta:\|\beta-\beta^*\|_2\leq R\}, \text{ where } R=k\|\beta^*\|_2.
\end{align}
\end{lemma}

\begin{lemma}\label{lemma:6}
With the same assumptions as in Lemma \ref{lemma:5}, for each $\beta\in \mathcal{B}$ with $\|\beta\|_0=s$,  $[\nabla q_i(\beta; \beta)]_j$ is $\xi$-sub-exponential with 
\begin{multline}
    \xi=C[(1+k)(1+kr)^2\sqrt{s}\|\beta^*\|_2+\max\{(1+kr)^2, \sigma^2+\|\beta^*\|_2^2\}]
\end{multline}
for some constant $C>0$.  Also, each $[\nabla q_i(\beta;\beta)]_j$, where $i\in [n]$, is independent of others for any fixed $j\in[d]$. 
\end{lemma}

\begin{theorem}
In an  $\epsilon$-corrupted high dimensional linear regression with missing covariates model with  $\epsilon$ satisfying the condition of 
\begin{multline*}
    [(1+k)(1+kr)^2\sqrt{s}\|\beta^*\|_2+ 
    \max\{(1+kr)^2, \sigma^2+\|\beta^*\|_2^2\}]\sqrt{s^*}(\epsilon\log (nd)+\sqrt{\frac{\log d}{n}})\\
    \leq O(\|\beta^*\|_2)
\end{multline*}
for some $k\in (0,1)$, if 
$\|\beta^{\text{init}}-\beta^*\|_2\leq \frac{k\|\beta^*\|_2^2}{2}$ and the assumptions in Lemma \ref{lemma:5} hold,
then, the output  $\beta^T$ of Algorithm \ref{alg:1} after taking $s=O(s^*)$ and $\eta=O(1)$ satisfies the following with probability at least $1-Td^{-3}$  
\begin{align}\label{eq:28}
  \|\beta^T-\beta^*\|_2 &\leq \gamma^{\frac{t}{2}}\|\beta^*\|_2 + O\big(   \max\{\|\beta^*\|^2_2+\sigma^2, 1, \sqrt{s^*}\|\beta^*\|_2\}\nonumber \\
  & \times   \sqrt{s^*}(\epsilon \log(nd)+\sqrt{\frac{\log d}{n}})\big),
\end{align}
where the Big-$O$ term hides the terms of $k$ and $r$.
\end{theorem}

Note that similar to the mixture of regressions model, when $\epsilon\leq \tilde{O}(\frac{1}{\sqrt{n}})$,  the estimation error is $O(s^*\sqrt{\frac{\log d}{n}})$, which is only a factor of $\sqrt{s^*}$ away from the optimal. However, unlike the previous two models,  we  assume here that  SNR is upper bounded by some constant which is unavoidable as pointed out in \citep{loh2011high}.
\section{Experiments}

In this section, we empirically study the performance of Algorithm~\ref{alg:1} on the three models mentioned in the previous section. Since in the paper we mainly focus on the statistical setting and its theoretical behaviors, thus, we will only perform our algorithm on the synthetic data. It is notable that previous papers on the statistical guarantees of EM algorithm all perform their algorithms on synthetic data only such as \citep{balakrishnan2017statistical,wang2015high,yi2015regularized}. Thus, performing experiments on synthetic data only is enough for the paper.

For each of these models, we generate synthesized datasets according to the underlying distribution. We will use $\|\beta-\beta^*\|_2$ to measure the estimation error, and test how it is affected by different parameter settings from two aspects. Firstly, we examine how the \emph{underlying sparsity} parameter $s^*$ of the model affects the estimation error and whether it is consistent with our theoretical results. Secondly, we test how the corruption fraction $\epsilon$ of the data and the dimensionality $d$ affect the convergence rate, as well as the estimation error.  For each experiment, the data is corrupted as follows: We first randomly choose $\epsilon$ fraction of the input data, then we add a Gaussian noise for each of these data samples. The noise is sampled from a multivariate Gaussian distribution  $\mathcal{N}(0, 50 \|X\|_\infty  I_d)$.  All experiments are repeated for 20 runs and the average results are reported.

\vspace{0.1in}
\noindent{\bfseries Parameter setting } Throughout the experiments we will follow the setting of the previous related works on high dimensional EM algorithms which have statistical guarantees but are not corruption-proofing \citep{zhu2017high,wang2015high,yi2015regularized}. We fix the dataset size $n$ to be $2000$, because using a larger $n$ does not exhibit significant difference. For each model, the experiment is divided into three parts as mentioned previously: The first one (Figure~\ref{fig:err-vs-s}) measures $\|\beta-\beta^*\|_2$ v.s. $\sqrt{n/(s^*\log d)}$ by varying $s^*$ from $3$ to $15$, with $d$ fixed to be $100$, which follows the previous works \citep{wang2015high,zhu2017high}; The second one (Figure~\ref{fig:err-vs-eps}) examines the convergence behavior under different corruption rate $\epsilon$ which varies from $0$ to $0.2$; The last one (Figure~\ref{fig:err-vs-d}) shows the convergence behavior under different data dimensionality $d$ which ranges from $80$ to $240$, with fixed $\epsilon=0.2$. 

For each experiment, instead of choosing the initial vectors which are close to the optimal ones, we use random initialization. We will  set $s=s^*$ in our algorithm,  which is also used in the previous methods. Besides the parameter $s$, there are also two other parameters of the algorithm that need to be specified: the D-Trim parameter $\alpha$ and the step size $\eta$. We are also required to set the "noise level" for each of the three models, which is quantified by  $\sigma$ in their definitions. It is notable that the choices of these parameters are quite flexible.
\begin{description}
    \item[GMM]: Corrupted Gaussian Mixture Model~\eqref{eq:4}. We fix $\sigma$ to  $0.5$, $\alpha$ to $0.2$ and $\eta$ to $0.1$.
    \item[MRM] Corrupted Mixture of Regressions Model~\eqref{eq:6}. We fix $\sigma$ to $0.2$, $\alpha$ to $0.2$ and $\eta$ to $0.1$.
    \item[RMC] Corrupted Linear Regression with Missing Covariates Model~\eqref{eq:8}. We set $\sigma=0.1$,  $\alpha=0.3$, and the missing probability $p_m=0.1$, but use three different step sizes $\eta=0.05, 0.1, 0.08$ for the three parts of the experiment, respectively. 
\end{description}

\noindent{\bfseries Results } 
 Firstly, we will mainly show that the classical high dimensional gradient EM algorithm in \citep{wang2015high} is not robust against to the corruptions.  Here we conduct the algorithm on the three models. For each experiment, we tune the parameters to be optimal as showed in \citep{wang2015high}. We test the algorithm w.r.t to $\sqrt{n/(s^*\log d)}$, iteration and different dimensions $d$.
 
 As we can see from Figure \ref{fig:method_non}. In all the three models, the algorithm performs quite well if there is no corruptions ($\epsilon=0$) which also has been showed in the previous papers \citep{wang2015high,zhu2017high}. However, when there are $\epsilon=0.05$ fraction of the samples are corrupted, the classical high dimensional EM algorithm will achieve a large estimation error. These results motivate us to design some robust high dimensional EM algorithms while also have provable statistical guarantees.  

\begin{figure*}[!htbp]
\vspace{-0.1in}
    \centering
    \begin{subfigure}[b]{1\textwidth}
    \includegraphics[width=\textwidth]{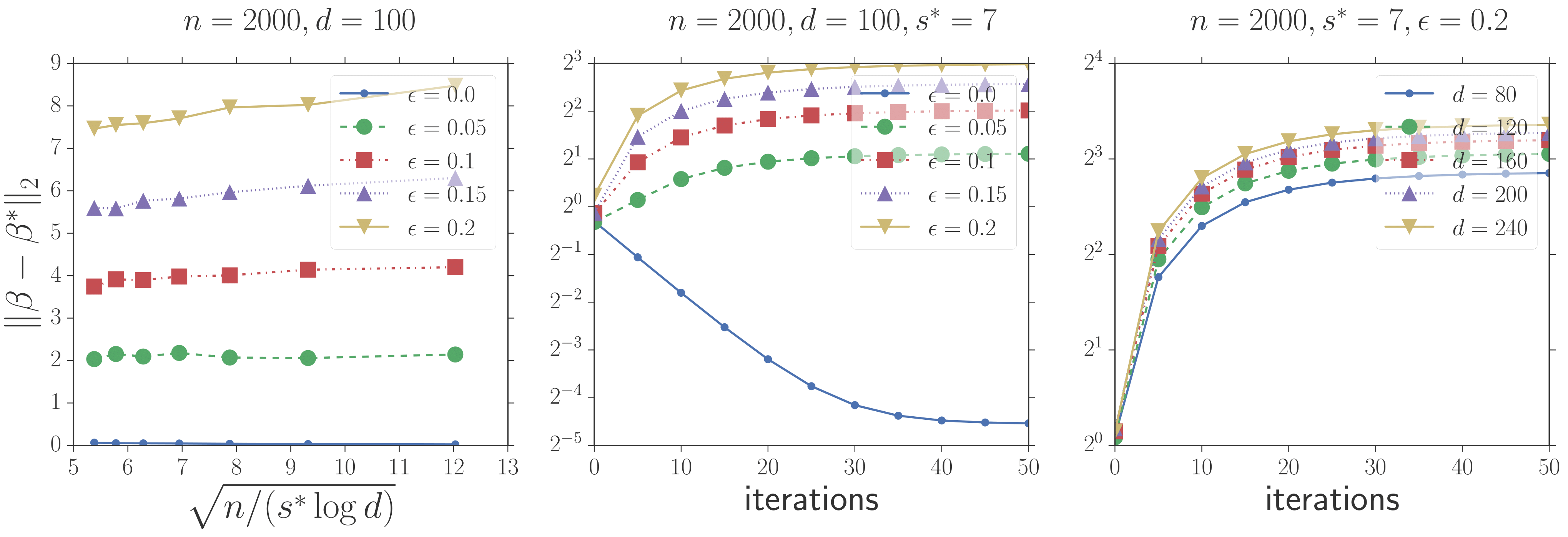}
    \caption{GMM\label{fig:GMM-err-vs-d_non}}
    \end{subfigure}
    ~
    \begin{subfigure}[b]{1\textwidth}
    \includegraphics[width=\textwidth]{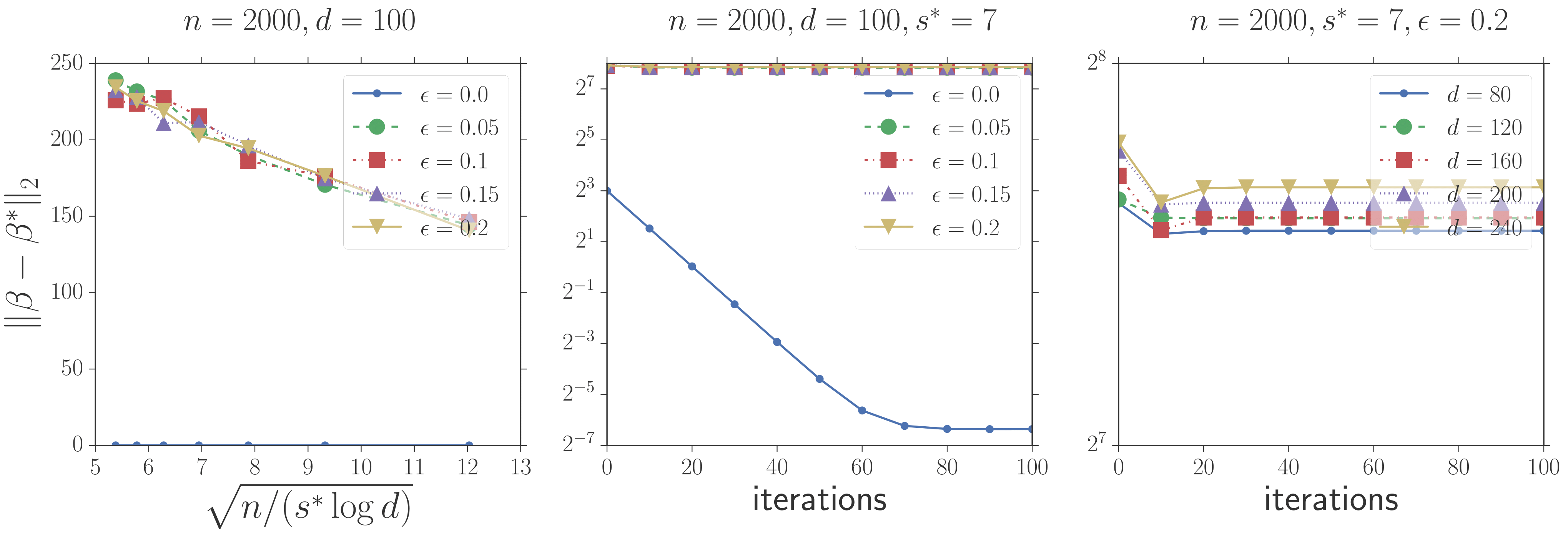}
    \caption{MRM\label{fig:MRM-err-vs-d_non}}
    \end{subfigure}
    ~
    \begin{subfigure}[b]{1\textwidth}
    \includegraphics[width=\textwidth]{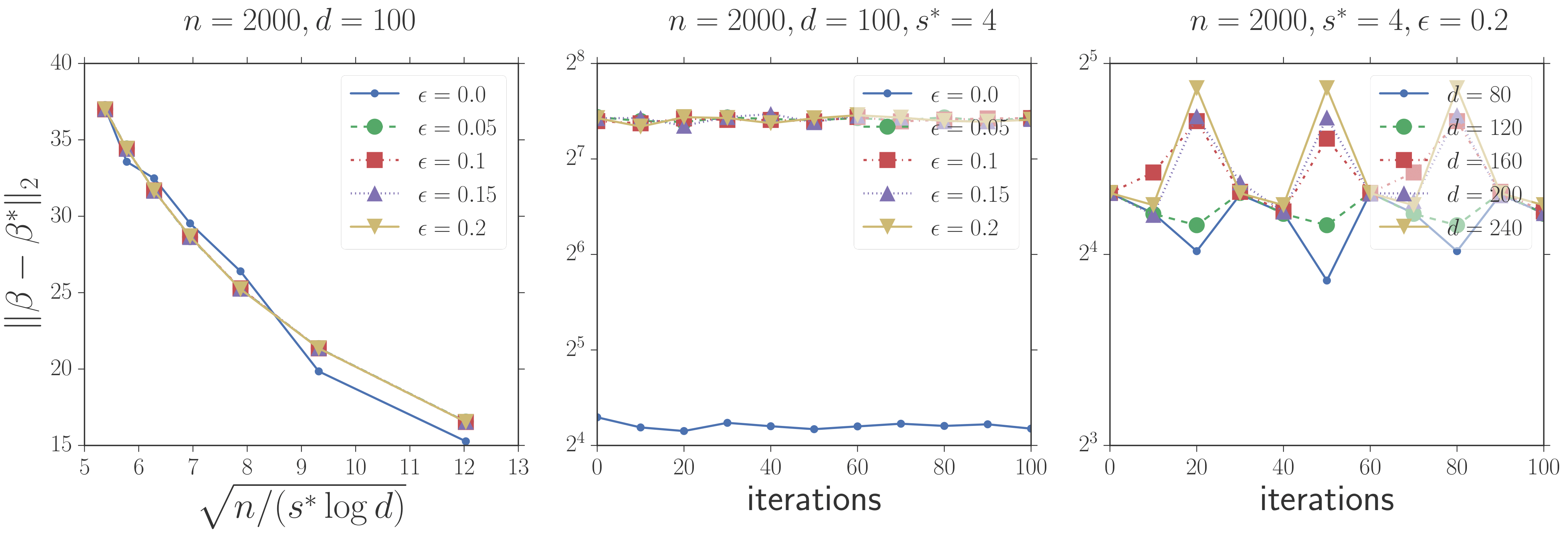}
    \caption{RMC\label{fig:RMC-err-vs-d_non}}
    \end{subfigure}

    \caption{Estimation error of classical high dimensional gradient EM algorithm in  \citep{wang2015high} w.r.t sample size, iteration and dimension.\label{fig:method_non}}
\end{figure*}

Next, we show the performance of our Algorithm \ref{alg:1}. For the first part (Figure~\ref{fig:err-vs-s}), we can see that when $\epsilon$ is small,  the final estimation error in each of the three models decreases when the term $\sqrt{n/(s^*\log d)}$ increases, as predicted by Theorem~\ref{thm:2}. But when $\epsilon$ is relatively large, the trend becomes less obvious for the Gaussian Mixture Model and the Mixture of Regressions model, because now the factor $\epsilon\log(nd)$ comes into play. 

Figure~\ref{fig:err-vs-eps} shows that our algorithm achieves linear convergence on all three models and all values of $\epsilon$, but the final converged error is heavily affected by $\epsilon$, and especially 
for the Gaussian Mixture and Linear Regression with Missing Covariates Models.  Moreover, when $\epsilon$ is small, the estimation errors are comparable to or even the same as  the non-corrupted ones, this is actually reasonable since it is corruption-proofing when $\epsilon$ is small theoretically. In the third part of the experiments (Figure~\ref{fig:err-vs-d}), varying $d$ seems not affect the convergence behavior much, which is reasonable as the error bound depends on $d$ only logarithmically and changes fairly slow. Thus, these results support Theorem \ref{thm:1}.

All the results show that our algorithm is robust against to some level of corruption while also could achieve an estimation error that is comparable to the non-corrupted ones. 
\vspace{-0.1in}
\section{Conclusion}
\vspace{-0.1in}
In this paper we  study the problem of estimating latent variable models with arbitrarily corrupted samples in the high dimensional sparse case and propose a method called Trimmed Gradient Expectation Maximization. Specifically, we show that our algorithm is  corruption-proofing and could achieve the (near) optimal statistical rate  for some statistical models  under some levels of corruption. Experimental results support our theoretical analysis and also show that our algorithm is indeed robust against to some corrupted samples. 

There are still many open problems. Firstly, in this paper, all of our theoretical guarantees need the initial parameter be close enough to the underlying parameter, which is quite strong. So how do we relax this assumption? Second, the three specific models we considered in the paper are quite simple, can we generalize to more models such as multi-component Gaussian Mixture Model or Mixture of Linear Regressions Model? Thirdly, in this paper we assume that the sparsity of the underlying parameter is known, how to deal with the case where it is unknown? 
\begin{figure*}[!htbp]
    \centering
    \begin{subfigure}[b]{.3\textwidth}
    \includegraphics[width=\textwidth,height=0.19\textheight]{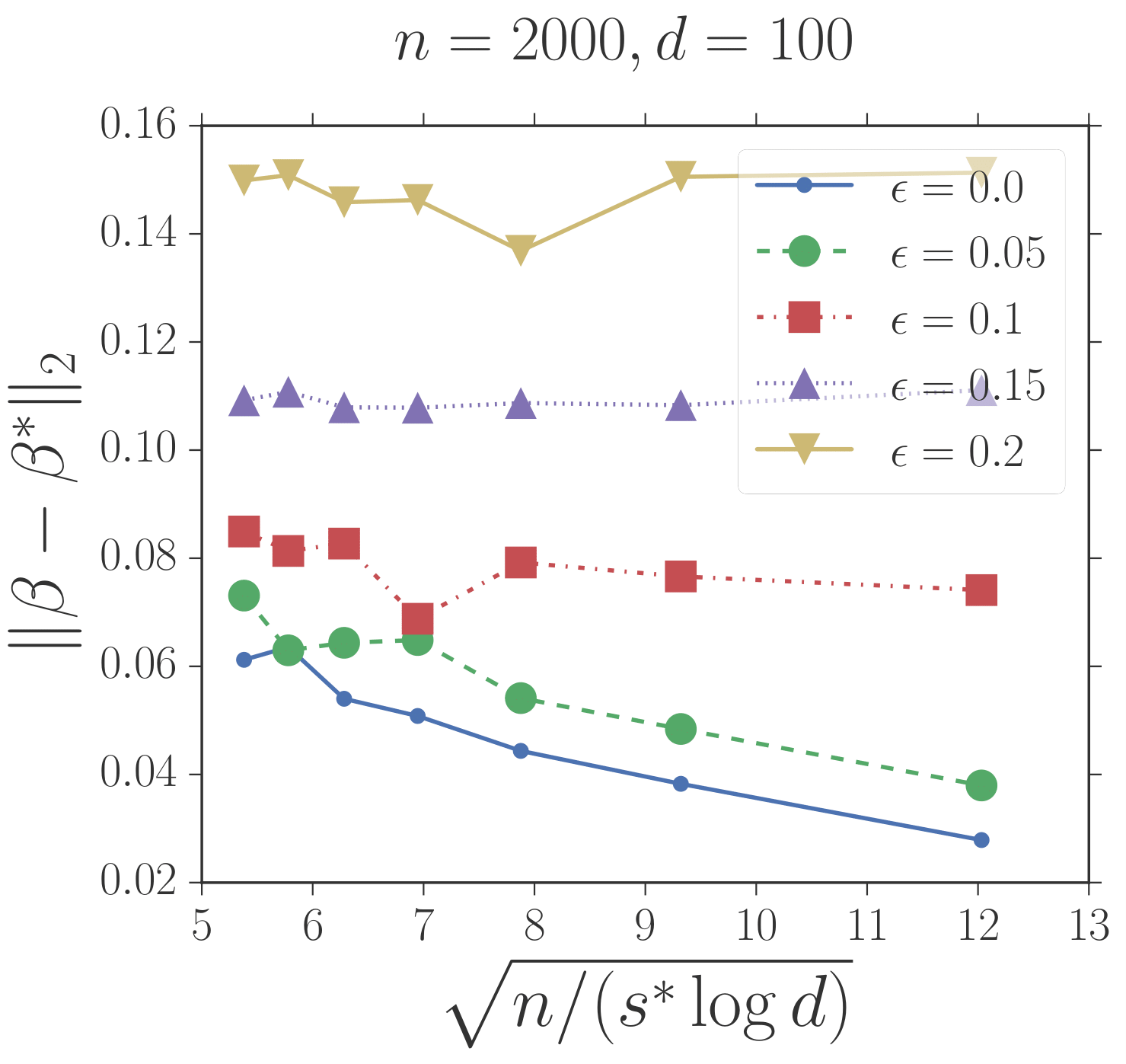}
    \caption{GMM\label{fig:GMM-err-vs-s}}
    \end{subfigure}
    ~
    \begin{subfigure}[b]{.3\textwidth}
    \includegraphics[width=\textwidth,height=0.19\textheight]{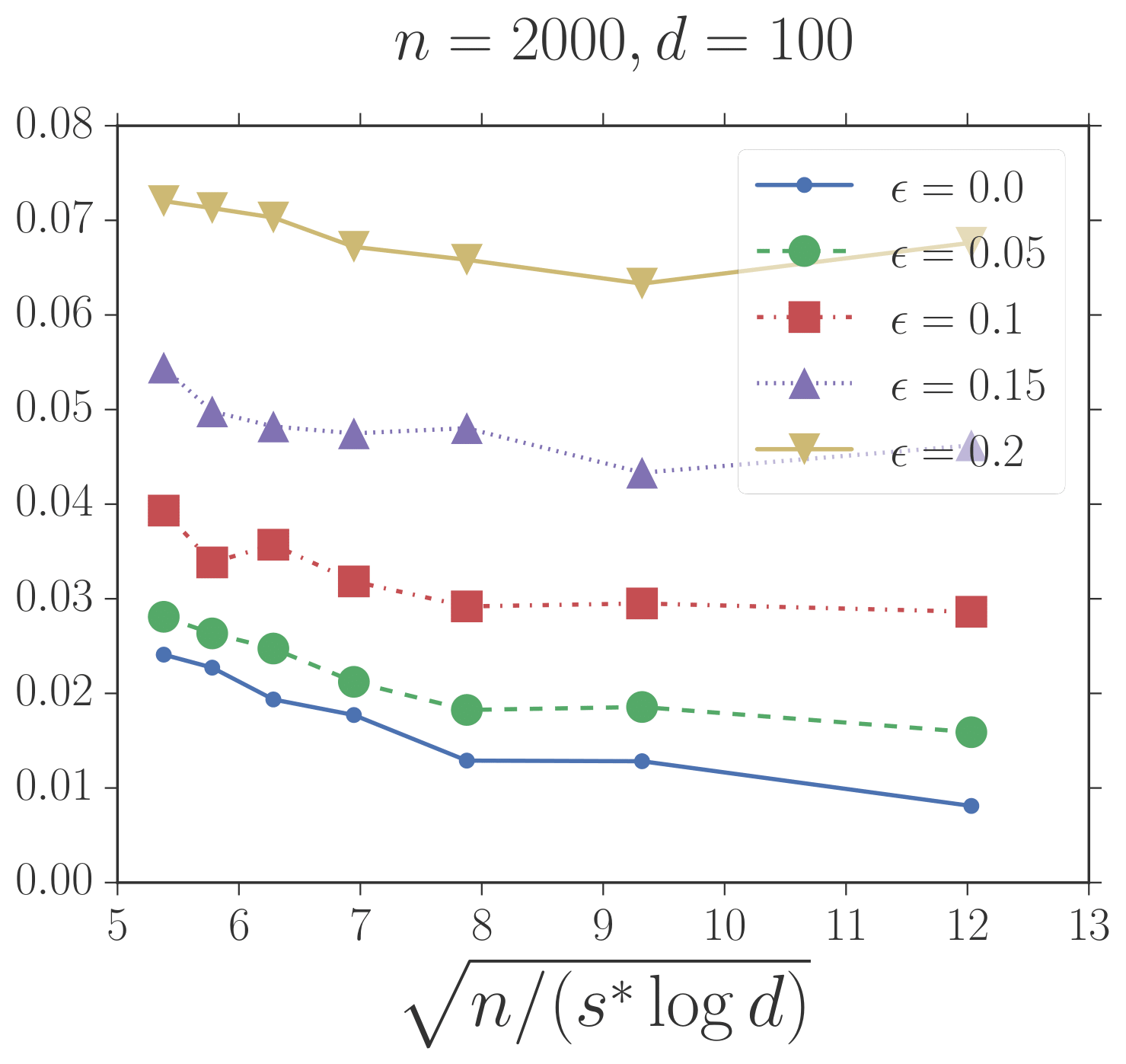}
    \caption{MRM\label{fig:MRM-err-vs-s}}
    \end{subfigure}
    ~
    \begin{subfigure}[b]{.3\textwidth}
    \includegraphics[width=\textwidth,height=0.19\textheight]{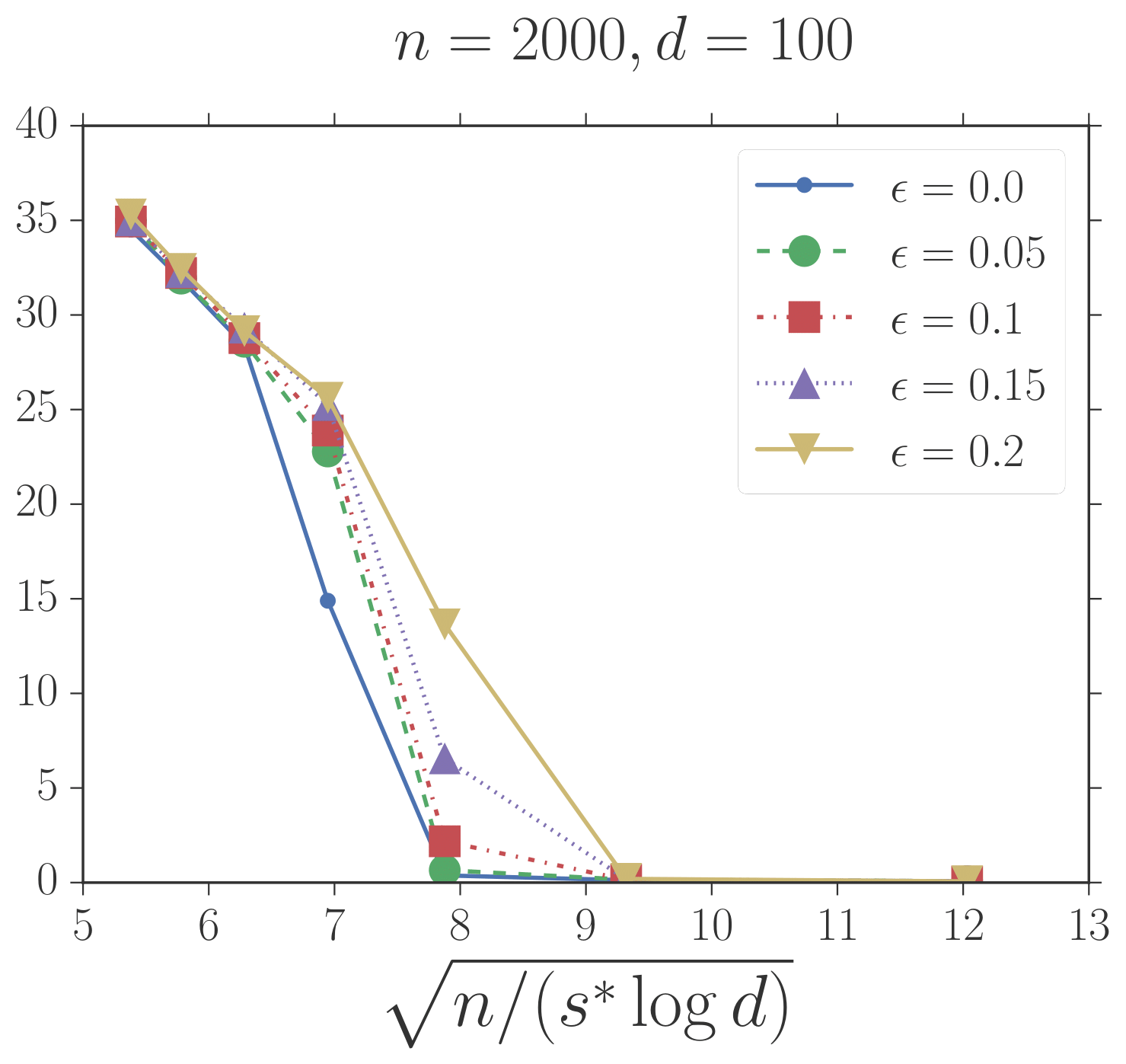}
    \caption{RMC\label{fig:RMC-err-vs-s}}
    \end{subfigure}

    \caption{Estimation error v.s. $\sqrt{n/(s^*\log d)}$\label{fig:err-vs-s}}
\end{figure*}

\begin{figure*}[!htbp]
    \centering
    \begin{subfigure}[b]{.3\textwidth}
    \includegraphics[width=\textwidth,height=0.19\textheight]{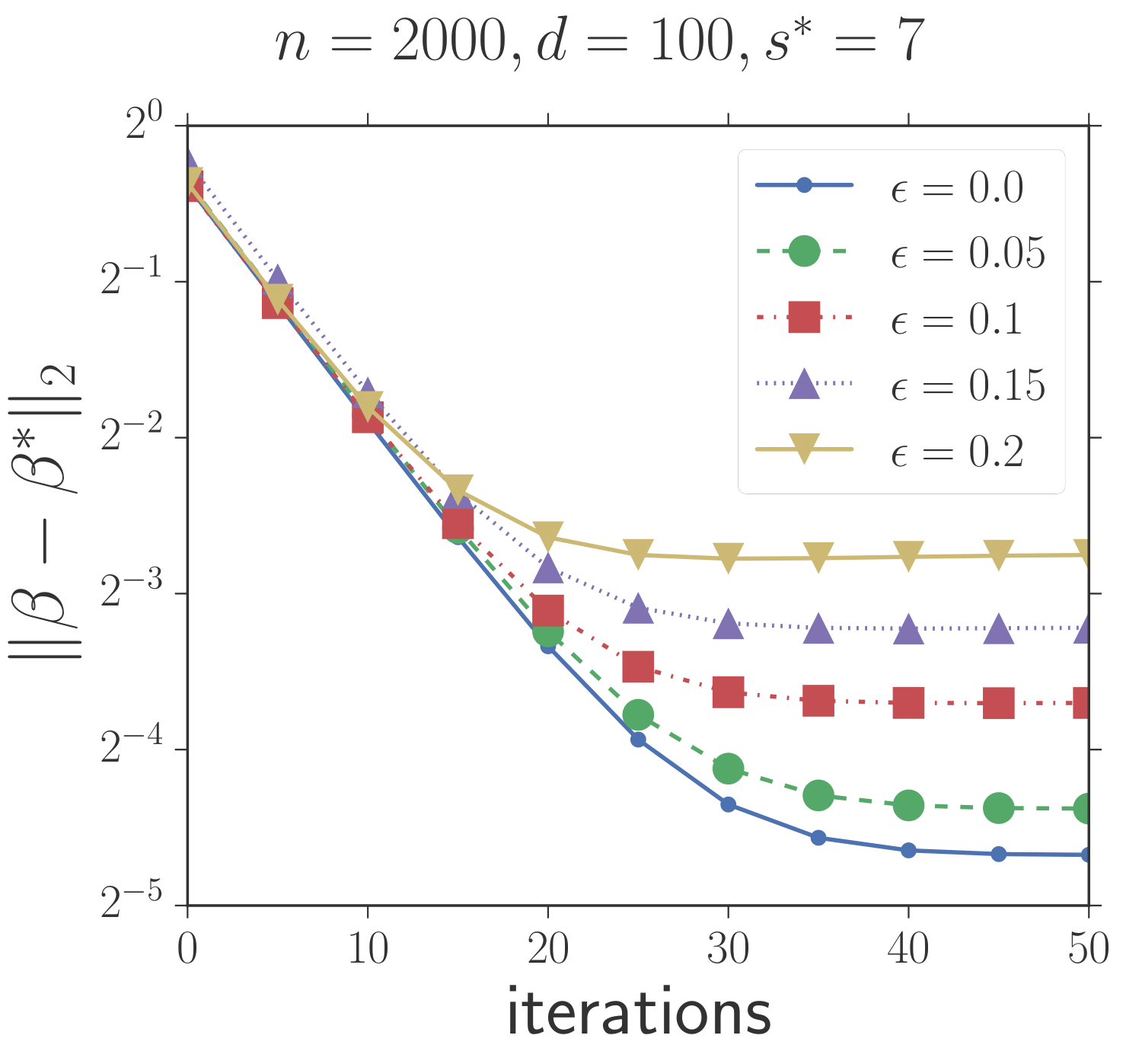}
    \caption{GMM\label{fig:GMM-err-vs-eps}}
    \end{subfigure}
    ~
    \begin{subfigure}[b]{.3\textwidth}
    \includegraphics[width=\textwidth,height=0.19\textheight]{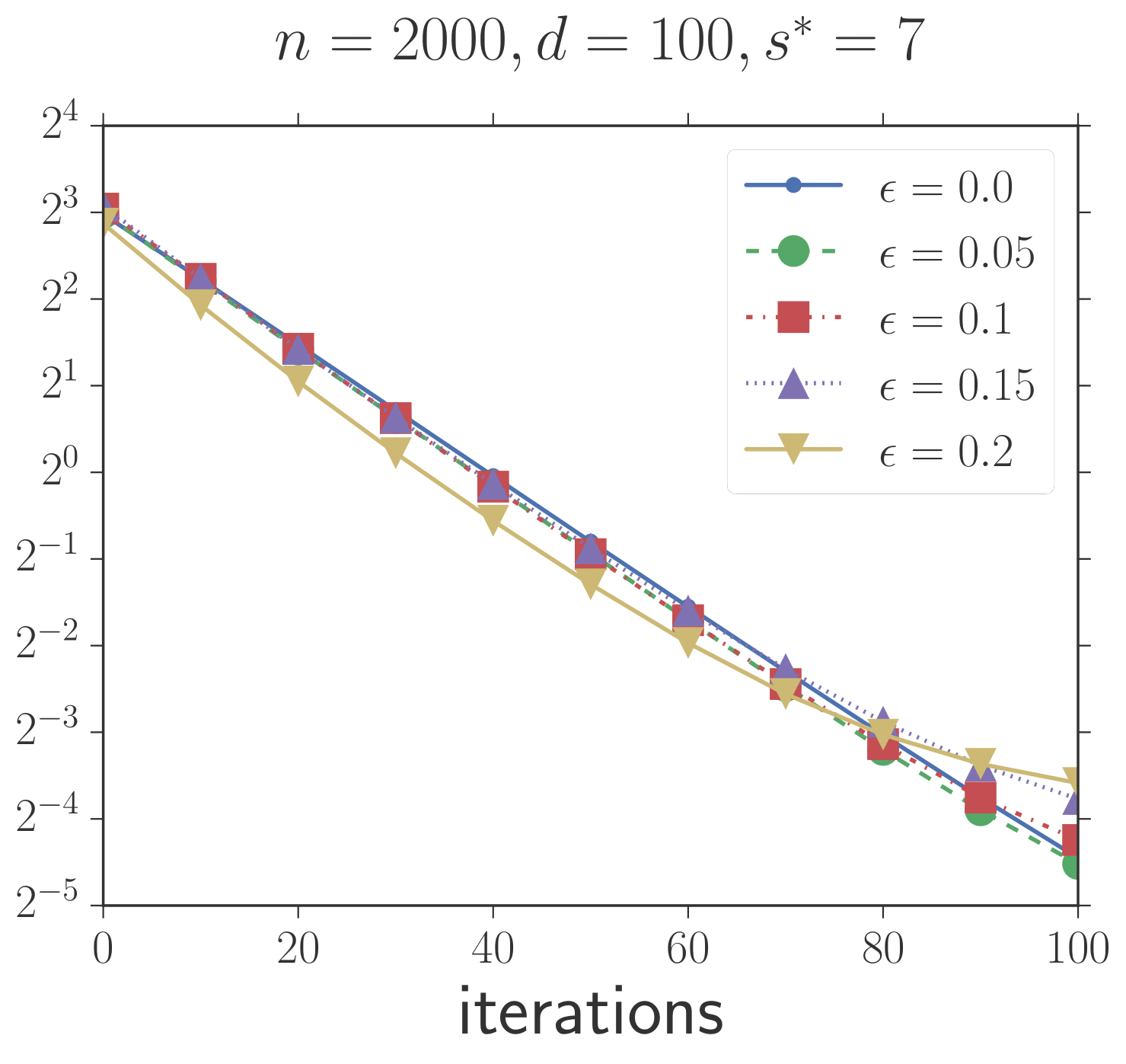}
    \caption{MRM\label{fig:MRM-err-vs-eps}}
    \end{subfigure}
    ~
    \begin{subfigure}[b]{.3\textwidth}
    \includegraphics[width=\textwidth,height=0.19\textheight]{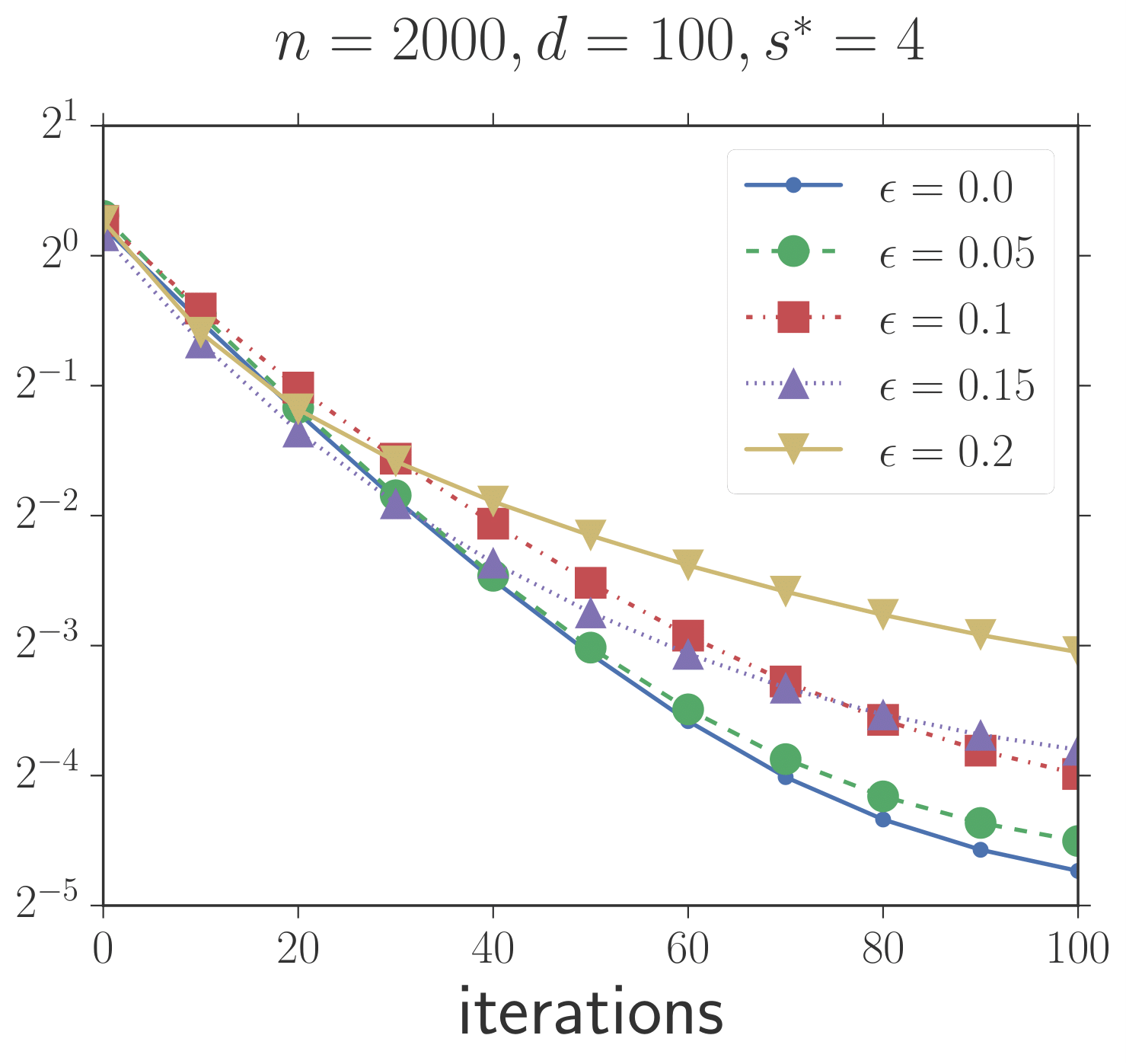}
    \caption{RMC\label{fig:RMC-err-vs-eps}}
    \end{subfigure}
    \caption{Estimation error v.s. iterations $t$ under different corruption rate $\epsilon$\label{fig:err-vs-eps}}
\end{figure*}
\begin{figure*}[!htbp]
    \centering
    \begin{subfigure}[b]{0.3\textwidth}
    \includegraphics[width=\textwidth,height=0.19\textheight]{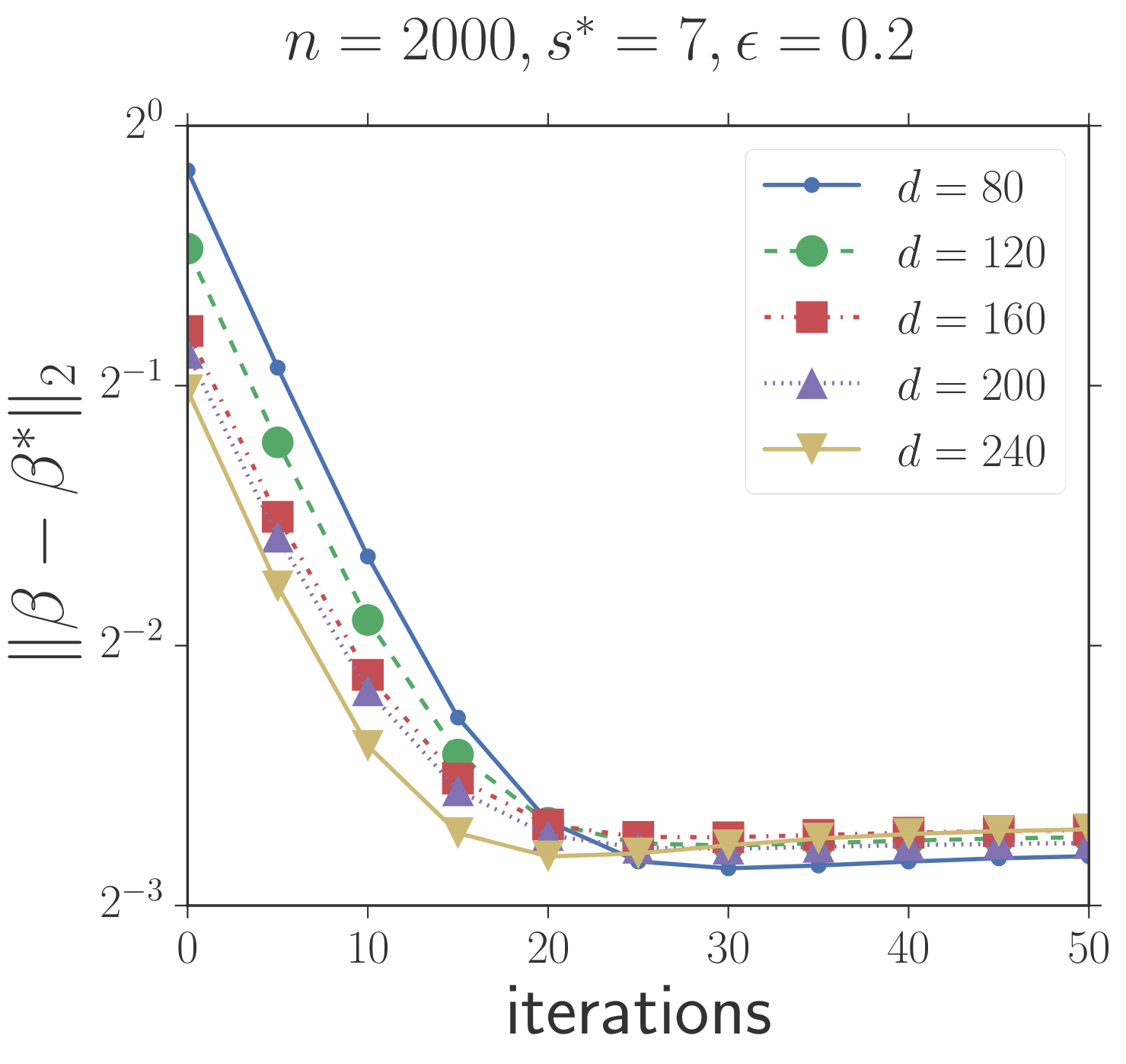}
    \caption{GMM\label{fig:GMM-err-vs-d}}
    \end{subfigure}
    ~
    \begin{subfigure}[b]{0.3\textwidth}
    \includegraphics[width=\textwidth,height=0.19\textheight]{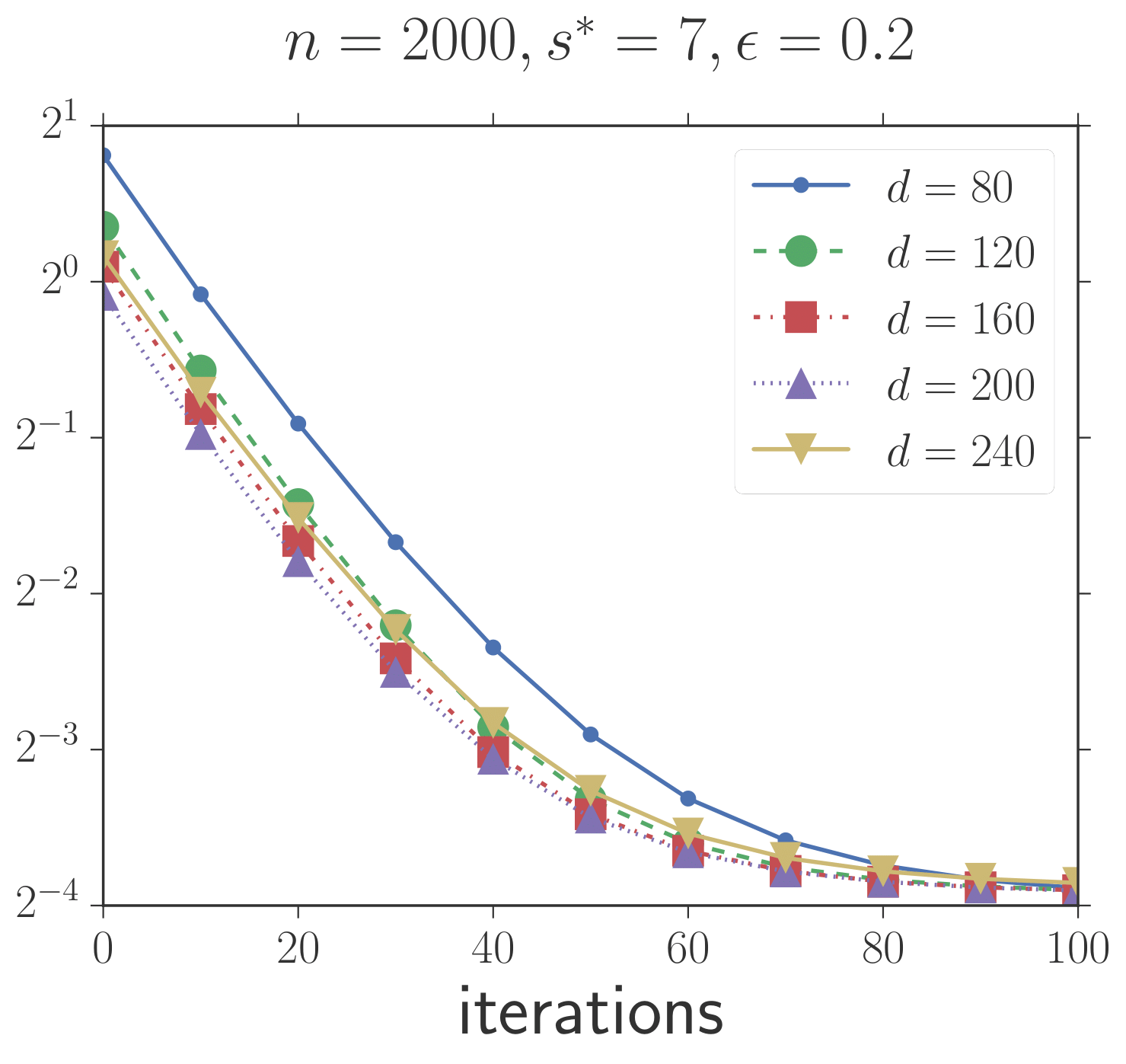}
    \caption{MRM\label{fig:MRM-err-vs-d}}
    \end{subfigure}
    ~
    \begin{subfigure}[b]{0.3\textwidth}
    \includegraphics[width=\textwidth,height=0.19\textheight]{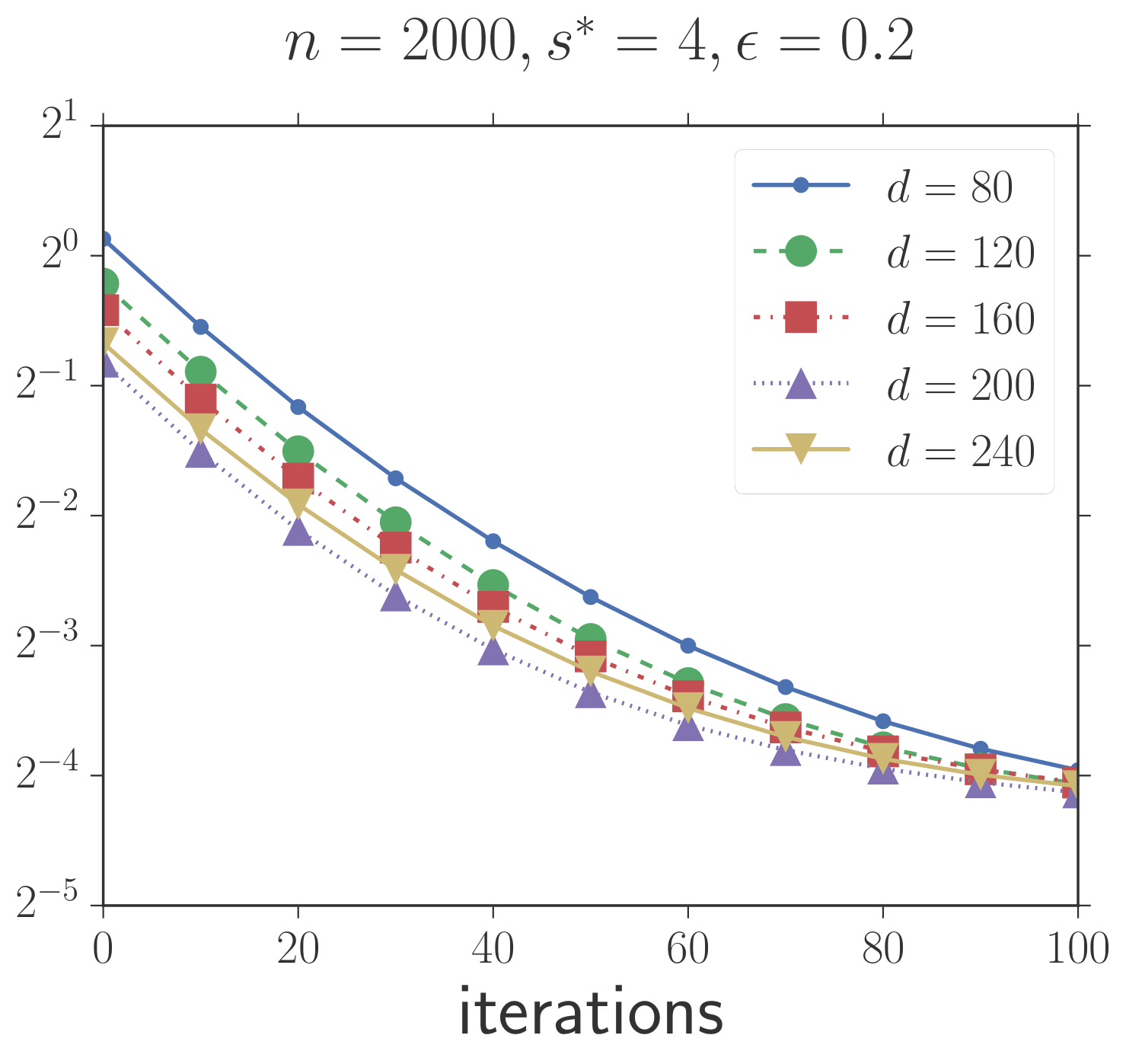}
    \caption{RMC\label{fig:RMC-err-vs-d}}
    \end{subfigure}

    \caption{Estimation error v.s. iterations $t$ under different dimensionality $d$\label{fig:err-vs-d}}
\end{figure*}


%
%

\bibliographystyle{spbasic}      
\bibliography{uai}   
\appendix
\section{Auxiliary Lemmas}
In this section, we introduce prerequisite knowledge and technical lemmas in order to prove the
main results.

In order to analyze the Dimensional $\alpha$-trimmed estimator, we first give some results for 1-dimensional samples and denote it as $\text{trmean}_\alpha(\cdot)$.
\begin{definition}\label{adef:1}
Given a set of $\epsilon$-corrupted samples $\{z_i\}_{i=1}^n\subseteq \mathbb{R}$, the trimmed mean estimator $\text{trmean}_\alpha(\{z_i\}_{i=1}^n)\in \mathbb{R}$ removes the largest and smallest $\alpha$ fraction of elements in $\{z_i\}_{i=1}^n$ and calculate the mean of the remaining terms. We choose $\alpha=c_0\epsilon$, for some constant $c_0\geq 1$. We also require that $\alpha\leq \frac{1}{2}-c_1$ for some small constant $c_1>0$. 
\end{definition}

For the 1-dimensional trimmed mean estimator, we have the following upper bound on the error w.r.t the population mean.

\begin{lemma}[Lemma A.2 in \citep{liu2019high}]\label{alemma:1}
Let $\{z_i\}_{i=1}^n\subset \mathbb{R}^d$ be 
$n=\Omega(\log d)$ $\epsilon$-corrupted samples. If 
the $j$-th coordinate, for each  $j\in[d]$, of the samples $\{z_{i, j}\}_{i=1}^n$ are i.i.d. $\xi$-exponential with mean $\mu^j$, then after using the dimensional $\alpha$-trimmed mean estimator, 
the following upper bound of error holds with probability at least $1-d^{-3}$, for every $j\in [d]$
\begin{equation}\label{aeq:1}
    |\text{trmean}_\alpha(\{z_{i,j}\}_{i=1}^n)-\mu^j|\leq C_2 \xi(\epsilon\log(nd)+\sqrt{\frac{\log d}{n}}),
    \end{equation}
    where $C_2$ is some constant dependent on $c_1$. 
\end{lemma}
Next, we provide some symmetrization results of random variables, which will be used in our proofs. See \citep{boucheron2013concentration} for details. 
\begin{lemma}\label{alemma:2}
Let $y_1, y_2, \cdots, y_n$ be the $n$ independent realizations of the random vector $Y\in \mathcal{Y}$, and $\mathcal{F}$ be a function class defined on $\mathcal{Y}$. For any increasing convex function $\phi(\cdot)$, the following holds  
\begin{equation*}
    \mathbb{E}\{\phi[\sup_{f\in \mathcal{F}}|\sum_{i=1}^n f(y_i)-\mathbb{E}(f(Y))|]\}\leq \mathbb{E}\{\phi[\sup_{f\in \mathcal{F}}|\sum_{i=1}^n \epsilon_i
f(y_i)|]\},
\end{equation*}
where $\epsilon_1, \cdots, \epsilon_n$ are i.i.d. Rademacher random variables that are independent of $y_1, \cdots, y_n$. 
\end{lemma}

\begin{lemma}\label{alemma:3}
Let $y_1, \cdots, y_n$ be $n$ independent realization of the random vector $Z\in \mathcal{Z}$ and $\mathcal{F}$ be a function class defined on $\mathcal{Z}$. 
If 
Lipschitz functions $\{\phi_i(\cdot)\}_{i=1}^n$ satisfy the following for all $v, v'\in\mathbb{R}$
\begin{equation*}
    |\phi_i(v)-\phi_i(v')|\leq L|v-v'| 
\end{equation*}
and $\phi_i(0)=0$, then for any increasing convex function $\phi(\cdot)$, the following holds 
\begin{equation*}
    \mathbb{E}\{\phi[|\sup_{f\in \mathcal{F}}\sum_{i=1}^n\epsilon_i\phi_i(f(y_i))|]\}\leq \mathbb{E}\{\phi[2|L\sup_{f\in \mathcal{F}}\sum_{i=1}^n\epsilon_if(y_i)|]\},
\end{equation*}
where $\epsilon_1, \cdots, \epsilon_n$ are i.i.d. Rademacher random variables that are independent of $y_1, \cdots, y_n$. 
\end{lemma}

Finally we recall some definitions and lemmas on the sub-exponential and sub-Gaussian random variables. See \citep{vershynin2010introduction} for details. 
\begin{definition}\label{adef:2}
For a sub-exponential random vector $X$, 
its  sub-exponential norm $\|X\|_{\psi_1}$ is defined as  \begin{equation*}
    \|X\|_{\psi_1} = \sup_{p\geq 1} p^{-1}(\mathbb{E}|X|^p)^{\frac{1}{p}}.
\end{equation*}
\end{definition}
\begin{lemma}\label{alemma:4}
Let $X$ be a zero-mean sub-exponential random variable, then there are absolute constants $C, c>0$, such that when $|t|\leq \frac{c}{\|X\|_{\psi_1}}$ ,
\begin{equation*}
    \mathbb{E}[\exp(tX)]\leq \exp(Ct^2\|X\|_{\psi_1}^2).
\end{equation*}
\end{lemma}
\begin{lemma}[Bernstein's inequality]\label{alemma:5}
Let $X_1,\cdots, X_n$ be $n$ i.i.d. realizations of $\upsilon$-sub-exponential random variable $X$ with mean $\mu$. Then, 
\begin{equation*}
    \text{Pr}(|\frac{1}{n}\sum_{i=1}^n X_i-\mu|\geq t)\leq 2\exp(-n\min(-\frac{t^2}{\upsilon^2}, \frac{t}{2\upsilon})).
\end{equation*}
\end{lemma}
\begin{definition}\label{adef:3}
A random variable $X$ is sub-Gaussian with variance $\sigma^2$ if for all $t>0$, the following holds  
\begin{equation*}
    \text{Pr}(|X-\mathbb{E}X|\geq t)\leq 2\exp(-\frac{t^2}{2\sigma^2}).
\end{equation*}
\end{definition}
\begin{definition}\label{adef:4}
For a sub-Gaussian random variable $X$, its sub-Gaussian norm $\|X\|_{\psi_2}$  is defined as 
\begin{equation*}
    \|X\|_{\psi_2}= \sum_{p\geq 1}p^{-\frac{1}{2}}(\mathbb{E}|X|^p)^{\frac{1}{p}}. 
\end{equation*}
\end{definition}
\begin{lemma}\label{alemma6}
If $X$ is sub-Gaussian or sub-exponential, then  $\|X-\mathbb{E} X\|_{\psi_2}\leq 2\|X\|_{\psi_2}$ or $\|X-\mathbb{E} X\|_{\psi_1}\leq 2\|X\|_{\psi_1}$ holds, respectively. 
\end{lemma}
\begin{lemma}\label{alemma7}
For two sub-Gaussian random variables $X_1, X_2$, $X_1\cdot X_2$ is a sub-exponential random variable with 
\begin{equation*}
    \|X_1\cdot X_2\|_{\psi_1}\leq C\max\{\|X_1\|_{\psi_2}^2, \|X_2\|_{\psi_2}^2\}.
\end{equation*}
\end{lemma}
\begin{lemma}\label{alemma8}
Let $X_1, X_2, \cdots, X_k$ be $k$ independent zero-mean sub-Gaussian random variables, and $X= \sum_{j=1}^kX_j$. Then, $X$ is sub-Gaussian with  $\|X\|_{\psi_2}^2\leq C\sum_{j=1}^k\|X_j\|_{\psi_2}^2$ for some absolute constant $C>0$.
\end{lemma}
\section{Omitted Proofs}
\subsection{Proof of Theorem 1}
By Lemma \ref{alemma:1} and our assumption on the $\xi$-sub-exponential property of each coordinate, we have the following in the $t$-th iteration with probability at least $1-d^{-3}$ for some constant $C_2>0$
\begin{equation}\label{aeq:2}
 \|\nabla \tilde{Q}_n(\beta^t; \beta^t) - \nabla Q(\beta^t; \beta^t)\|_\infty \leq C_2 \xi(\epsilon\log(nd)+\sqrt{\frac{\log d}{n}}).
\end{equation}
For convenience, we let $\alpha=C_2 \xi(\epsilon\log(nd)+\sqrt{\frac{\log d}{n}})$, and  assume that for all iterations $t\in [T-1]$, event (\ref{aeq:2}) holds (then all events hold with probability at least $1-Tp^{-3}$). 

In the $t$-th iteration, we define  
\begin{equation}\label{aeq:3}
    \bar{\beta}^{t+0.5}= \beta^t+\eta \nabla Q(\beta^t; \beta^t)
\end{equation}
and 
\begin{equation}\label{aeq:4}
   \bar{\beta}^{t+1}=\text{trunc}(\bar{\beta}^{t+0.5}, \hat{\mathcal{S}}^{t+0.5}). 
\end{equation}
That is, $\bar{\beta}^{t+0.5}$ is the gradient update of $\beta^t$ w.r.t the non-corrupted population gradient of $Q_n(\beta^t; \beta^t)$, and $\bar{\beta}^{t+1}$ is the estimation after truncating $\bar{\beta}^{t+0.5}$ w.r.t set $\hat{\mathcal{S}}^{t+0.5}$, which is the set of the $s$-largest coordinates of $\beta^{t+0.5}$.

By the definition, we have the following inequalities
\begin{align}\label{aeq:5}
    &\|\beta^{t+1}-\beta^*\|_2=\|\text{trunc}(\beta^{t+0.5}, \hat{\mathcal{S}}^{t+0.5})-\beta^*\|_2 \nonumber \\
    &\leq \|\text{trunc}(\beta^{t+0.5}, \hat{\mathcal{S}}^{t+0.5})-\text{trunc}(\bar{\beta}^{t+0.5}, \hat{\mathcal{S}}^{t+0.5})\|_2 \nonumber +\|\text{trunc}(\bar{\beta}^{t+0.5}, \hat{\mathcal{S}}^{t+0.5})-\beta^*\|_2\\
    &= \|\text{trunc}(\beta^{t+0.5}, \hat{\mathcal{S}}^{t+0.5})-\text{trunc}(\bar{\beta}^{t+0.5}, \hat{\mathcal{S}}^{t+0.5})\|_2+\|\bar{\beta}^{t+1}-\beta^*\|_2\nonumber \\
    &\leq \underbrace{\|(\beta^{t+0.5}-\bar{\beta}^{t+0.5})_{\hat{\mathcal{S}}^{t+0.5}}\|_2}_{A}+ \underbrace{\|\bar{\beta}^{t+1}-\beta^*\|_2}_{B}.
\end{align}
For the term A, we have 
\begin{align}\label{aeq:6}
   \|(\beta^{t+0.5}-\bar{\beta}^{t+0.5})_{\hat{\mathcal{S}}^{t+0.5}}\|_2 &\leq \sqrt{s}\|\beta^{t+0.5}-\bar{\beta}^{t+0.5}\|_\infty\nonumber \\
   &=\eta \sqrt{s}\|\nabla \tilde{Q}_n(\beta^t; \beta^t)- \nabla Q(\beta^t; \beta^t)\|_\infty.
\end{align}
Thus, if $\beta^t\in \mathcal{B}$, {\em i.e.,} $\|\beta^t-\beta^*\|\leq k\|\beta^*\|_2$) and $\|\beta^t\|_0=s$, then by the assumption and (\ref{aeq:2}), we have 
\begin{equation}\label{aeq:7}
    A\leq \eta\sqrt{s}\alpha.
\end{equation}
Next, we will bound the term B. To do this, we need the following lemma, which follows \citep{wang2015high}. 
\begin{lemma}\label{alemma:6}
If 
\begin{equation}\label{aeq:8}
    \|\bar{\beta}^{t+0.5}-\beta^*\|_2\leq k \|\beta^*\|_2
\end{equation}
for some $k\in (0,1)$ and 
\begin{equation}\label{aeq:9}
    s\geq \frac{4(1+k)^2}{(1-k)^2}s^* \text{ and } \sqrt{s}\|\beta^{t+0.5}-\bar{\beta}^{t+0.5}\|_\infty\leq \frac{(1-k)^2}{2(1+k)}\|\beta^*\|_2, 
\end{equation}
then, the following holds  
\begin{equation}\label{aeq:10}
    \|\bar{\beta}^{t+1}-\beta^*\|_2\leq \frac{C\sqrt{s^*}}{\sqrt{1-k}}\|\beta^{t+0.5}-\bar{\beta}^{t+0.5}\|_\infty+(1+4\sqrt{\frac{s^*}{s}})^{1/2}\|\bar{\beta}^{t+0.5}-\beta^*\|_2.
\end{equation}
\end{lemma}
\begin{proof}[Proof of Lemma \ref{alemma:6}]
By  assumption (\ref{aeq:8}), we have 
\begin{equation}
    (1-k)\|\beta^*\|_2\leq \|\bar{\beta}^{t+0.5}\|_2\leq (1+k)\|\beta^*\|_2.
\end{equation}
We then denote 
\begin{equation}\label{aeq:12}
    \bar{\theta}=\frac{\bar{\beta}^{t+0.5}}{\|\bar{\beta}^{t+0.5}\|_2}, \theta=\frac{\beta^{t+0.5}}{\|\bar{\beta}^{t+0.5}\|_2} \text{, and } \theta^*= \frac{\beta^*}{\|\beta^*\|_2}
\end{equation}
and the sets $\mathcal{I}_1, \mathcal{I}_2$ and $\mathcal{I}_3$ as the follows
\begin{equation}\label{aeq:13}
    \mathcal{I}_1=\mathcal{S}^*\backslash \hat{\mathcal{S}}^{t+0.5}, \mathcal{I}_2=\mathcal{S}^*\bigcap \hat{\mathcal{S}}^{t+0.5} \text{, and }   \mathcal{I}_3=\hat{\mathcal{S}}^{t+0.5} \backslash \mathcal{S}^*,
\end{equation}
where $S^*=\text{supp}(\beta^*)$.
Let $s_i=|\mathcal{I}_i|$ for $i=1, 2, 3$, respectively. Also, we define $\Delta=\langle \bar{\theta}, \theta^*\rangle$. Note that 
\begin{equation}\label{aeq:14}
    \Delta=\langle \bar{\theta}, \theta^*\rangle= \sum_{j\in \mathcal{S}^*}\bar{\theta}_j\theta_j^*=\sum_{j\in \mathcal{I}_1}\bar{\theta}_j\theta_j^*+\sum_{j\in \mathcal{I}_2}\bar{\theta}_j\theta_j^*\leq \|\bar{\theta}_{\mathcal{I}_1}\|_2\|\theta^*_{\mathcal{I}_1}\|_2+\|\bar{\theta}_{\mathcal{I}_2}\|_2\|\theta^*_{\mathcal{I}_2}\|_2.
\end{equation}
By Cauchy-Schwartz inequality, we  have 
\begin{align}\label{aeq:15}
    \Delta^2&\leq (\|\bar{\theta}_{\mathcal{I}_1}\|_2\|\theta^*_{\mathcal{I}_1}\|_2+\|\bar{\theta}_{\mathcal{I}_2}\|_2\|\theta^*_{\mathcal{I}_2}\|_2)^2 \nonumber \\
    &\leq (\|\bar{\theta}_{\mathcal{I}_1}\|_2^2+\|\bar{\theta}_{\mathcal{I}_2}\|_2^2)(\|\theta^*_{\mathcal{I}_1}\|_2^2+\|\theta^*_{\mathcal{I}_2}\|_2^2)\nonumber\\
    &= (1-\|\bar{\theta}_{\mathcal{I}_3}\|_2^2)(1-\|\theta^*_{\mathcal{I}_3}\|_2^2)\leq 1-\|\bar{\theta}_{\mathcal{I}_3}\|_2^2.
\end{align}
Since $\mathcal{I}_3\subseteq \hat{\mathcal{S}}^{t+0.5}$ and $\mathcal{I}_1\bigcap \hat{\mathcal{S}}^{t+0.5}=\emptyset$, we have 
\begin{equation}\label{aeq:16}
    \frac{\|\beta^{t+0.5}_{\mathcal{I}_3}\|^2_2}{\|\beta^{t+0.5}_{\mathcal{I}_1}\|^2_2}\geq \frac{s_3}{s_1}, \text{ i.e., } \frac{\|\theta_{\mathcal{I}_3}\|_2}{\sqrt{s_3}}\geq \frac{\|\theta_{\mathcal{I}_1}\|_2}{\sqrt{s_1}}. 
\end{equation}
We let $\tilde{\epsilon}=2\|\bar{\theta}-\theta\|_\infty=2\frac{\|\bar{\beta}^{t+0.5}-\beta^{t+0.5}\|_\infty}{\|\bar{\beta}^{t+0.5}\|_2}$. Note that we have 
\begin{equation}\label{aeq:17}
    \max\{\frac{\|\theta_{\mathcal{I}_3}-\bar{\theta}_{\mathcal{I}_3}\|_2  }{\sqrt{s_3}}, \frac{\|\theta_{\mathcal{I}_1}-\bar{\theta}_{\mathcal{I}_1}\|_2  }{\sqrt{s_1}}\}\leq \max\{\|\theta_{\mathcal{I}_3}-\bar{\theta}_{\mathcal{I}_3}\|_\infty, \|\theta_{\mathcal{I}_1}-\bar{\theta}_{\mathcal{I}_1}\|_\infty \}\leq \|\bar{\theta}-\theta\|_\infty=\frac{\tilde{\epsilon}}{2},
\end{equation}
which implies that 
\begin{align}\label{aeq:18}
\frac{\|\bar{\theta}_{\mathcal{I}_3}\|_2}{\sqrt{s_3}}\geq \frac{\|\theta_{\mathcal{I}_3}\|_2}{\sqrt{s_3}}-\frac{\|\theta_{\mathcal{I}_3}-\bar{\theta}_{\mathcal{I}_3}\|_2 }{\sqrt{s_3}}& \underset{(a)}{\geq} \frac{\|\theta_{\mathcal{I}_1}\|_2}{\sqrt{s_1}}-\frac{\|\theta_{\mathcal{I}_3}-\bar{\theta}_{\mathcal{I}_3}\|_2 }{\sqrt{s_3}}\nonumber \\
&\geq \frac{\|\bar{\theta}_{\mathcal{I}_1}\|_2}{\sqrt{s_1}}-\frac{\|\theta_{\mathcal{I}_3}-\bar{\theta}_{\mathcal{I}_3}\|_2 }{\sqrt{s_3}}- \frac{\|\theta_{\mathcal{I}_1}-\bar{\theta}_{\mathcal{I}_1}\|_2 }{\sqrt{s_1}}\geq \frac{\|\bar{\theta}_{\mathcal{I}_1}\|_2}{\sqrt{s_1}}-\tilde{\epsilon},
\end{align}
where inequality (a) is due to (\ref{aeq:16}).
Plugging (\ref{aeq:18}) into (\ref{aeq:15}), we have 
\begin{equation}\label{aeq:19}
    \Delta^2\leq 1-\|\bar{\theta}_{\mathcal{I}_3}\|_2^2 \leq 1-(\sqrt{\frac{s_3}{s_1}}\|\bar{\theta}_{\mathcal{I}_1}\|_2-\sqrt{s_3}\tilde{\epsilon})^2.
\end{equation}
Solving $\|\bar{\theta}_{\mathcal{I}_1}\|_2$ in (\ref{aeq:19}), we  get 
\begin{equation}\label{aeq:20}
    \|\bar{\theta}_{\mathcal{I}_1}\|_2\leq \sqrt{\frac{s_1}{s_3}}\sqrt{1-\Delta^2}+\sqrt{s_1}\tilde{\epsilon}\leq \sqrt{\frac{s^*}{s}}\sqrt{1-\Delta^2}+\sqrt{s^*}\tilde{\epsilon}.
\end{equation}
The final inequality is due to the inequality   $\frac{s_1}{s_3}\leq \frac{s_1+s_2}{s_3+s_2}=\frac{s^*}{s}$, which follows from $\frac{s^*}{s}\leq \frac{(1-k)^2}{4(1+k)^2}\leq 1$ and $s_3\geq s-s^*\geq s^*\geq s_1$.

In the following, we will prove that the right hand side of (\ref{aeq:20}) is upper bounded by $\Delta$. To achieve this,
it is sufficient to show that 
\begin{align}\label{aeq:21}
    \Delta\geq \frac{\sqrt{s^*}\tilde{\epsilon}+[s^*\tilde{\epsilon}^2-(s^*/s+1)(s^*\tilde{\epsilon^2-s^*/s)]^{\frac{1}{2}}}}{s^*/s+1}=\frac{\sqrt{s^*}\tilde{\epsilon}
    +[-(s^*\tilde{\epsilon})^2/s+(s^*/s+1)s^*/s]^{\frac{1}{2}}
    }{s^*/s+1}.
\end{align}
To prove (\ref{aeq:21}), we first note that $\sqrt{s^*}\tilde{\epsilon}\leq \Delta$, which is due to 
\begin{equation}\label{aeq:22}
    \sqrt{s^*}\tilde{\epsilon}\leq \sqrt{s}\tilde{\epsilon}=\frac{2\sqrt{s}\|\bar{\beta}^{t+0.5}-\beta^{t+0.5}\|_\infty}{\|\beta^*\|_2}\leq \frac{1-k}{1+k}\leq \Delta,
\end{equation}
where the second inequality is due to  assumption (\ref{aeq:9}) and the final inequality is due to 
\begin{equation*}
    \Delta=\langle \bar{\theta}, \theta^*\rangle =\frac{\langle \bar{\beta}^{t+0.5}, \beta^*\rangle }{\|\bar{\beta}^{t+0.5}\|_2 \|\beta^*\|_2}\overset{(a)}{\geq} \frac{\|\bar{\beta}^{t+0.5}\|_2^2+\|\beta^*\|_2^2-k^2\|\beta^*\|_2^2}{2\|\bar{\beta}^{t+0.5}\|_2\|\beta^*\|_2}\geq \frac{(1-k)^2+1-k^2}{2(1+k)}=\frac{1-k}{1+k},
\end{equation*}
where inequality (a) is due to assumption (\ref{aeq:8}).

Now, we show that (\ref{aeq:21}) holds. By (\ref{aeq:22}), we have 
\begin{equation}
\sqrt{s}\tilde{\epsilon}\leq \frac{1-k}{1+k}<1<\sqrt{\frac{s^*+s}{s}},
\end{equation}
which implies that $\tilde{\epsilon}\leq \frac{\sqrt{s^*+s}}{s}$.

For the right hand side of (\ref{aeq:21}), we have 
\begin{align}
    \frac{\sqrt{s^*}\tilde{\epsilon}
    +[-(s^*\tilde{\epsilon})^2/s+(s^*/s+1)s^*/s]^{\frac{1}{2}}
    }{s^*/s+1}&\leq \frac{\sqrt{s^*}\tilde{\epsilon}
    +[(s^*/s+1)s^*/s]^{\frac{1}{2}}
    }{s^*/s+1}\\
    &\leq 2\sqrt{\frac{s^*}{s^*+s}}\leq 2\sqrt{\frac{1}{1+4(1+k)^2/(1-k)^2}}\\
    &\leq \frac{1-k}{1+k}\leq \Delta.
\end{align}
Thus, in total, by (\ref{aeq:20}) we can get 
\begin{equation}
    \|\bar{\theta}_{\mathcal{I}_1}\|_2\leq \Delta.
\end{equation}
From (\ref{aeq:15}), we can see that 
\begin{equation*}
     \Delta\leq \|\bar{\theta}_{\mathcal{I}_1}\|_2\|\theta^*_{\mathcal{I}_1}\|_2+\|\bar{\theta}_{\mathcal{I}_2}\|_2\|\theta^*_{\mathcal{I}_2}\|_2 \leq \|\bar{\theta}_{\mathcal{I}_1}\|_2\|\theta^*_{\mathcal{I}_1}\|_2+\sqrt{(1-\|\bar{\theta}_{\mathcal{I}_1}\|_2^2)}\sqrt{(1-\theta^*_{\mathcal{I}_1}\|^2_2)},
\end{equation*}
that is, 
\begin{equation*}
   (\Delta- \|\bar{\theta}_{\mathcal{I}_1}\|_2\|\theta^*_{\mathcal{I}_1}\|_2)^2 \leq (1-\|\bar{\theta}_{\mathcal{I}_1}\|_2^2)(1-\theta^*_{\mathcal{I}_1}\|^2_2).
\end{equation*}
Solving the above inequality, we get 
\begin{multline}\label{aeq:28}
    \|\theta^*_{\mathcal{I}_1}\|_2\leq \|\bar{\theta}_{\mathcal{I}_1}\|_2\Delta+\sqrt{1-\|\bar{\theta}_{\mathcal{I}_1}\|_2^2}\sqrt{1-\Delta^2}\leq \|\bar{\theta}_{\mathcal{I}_1}\|_2+\sqrt{1-\Delta^2} \\ 
    \leq \sqrt{\frac{s^*}{s}}\sqrt{1-\Delta^2}+\sqrt{s^*}\tilde{\epsilon}+ \sqrt{1-\Delta^2},
\end{multline}
where the final inequality is due to (\ref{aeq:20}). Combining this with (\ref{aeq:20}) and (\ref{aeq:28}), we have
\begin{equation}\label{aeq:29}
     \|\theta^*_{\mathcal{I}_1}\|_2  \|\bar{\theta}_{\mathcal{I}_1}\|_2\leq [\sqrt{\frac{s^*}{s}}\sqrt{1-\Delta^2}+\sqrt{s^*}\tilde{\epsilon}+ \sqrt{1-\Delta^2}]\cdot [\sqrt{\frac{s^*}{s}}\sqrt{1-\Delta^2}+\sqrt{s^*}\tilde{\epsilon}].
\end{equation}
Now, by the definition of $\bar{\theta}$, we have 
\begin{equation}
    \bar{\beta}^{t+1}=\text{trunc}(\bar{\beta}^{t+0.5}, \hat{\mathcal{S}}^{t+0.5})=\text{trunc}(\bar{\theta}, \hat{\mathcal{S}}^{t+0.5})\|\hat{\beta}^{t+0.5}\|_2.
\end{equation}
Therefore, we get
\begin{equation}\label{aeq:31}
    \langle \frac{\bar{\beta}^{t+1}}{\|\bar{\beta}^{t+0.5}\|_2}, \frac{\beta^*}{\|\beta^*\|_2} \rangle= \langle \text{trunc}(\bar{\theta}, \hat{\mathcal{S}}^{t+0.5}), \theta^* \rangle =\langle \bar{\theta}_{\mathcal{I}_2}, \theta^*_{\mathcal{I}_2}\rangle \geq \langle \bar{\theta}, \theta^*\rangle-\|\bar{\theta}_{\mathcal{I}_1}\|_2\|\theta^*_{\mathcal{I}_1}\|_2.
\end{equation}
Let $\chi=\|\bar{\beta}^{t+0.5}\|_2\|\beta^*\|_2$. Then, by (\ref{aeq:31}) and (\ref{aeq:29}) we have 
\begin{align}
   & \langle \bar{\beta}^{t+1}, \beta^*\rangle \nonumber \\
   &\geq \langle \bar{\beta}^{t+0.5}, \beta^*\rangle- [(\sqrt{\frac{s^*}{s}}+1)\sqrt{\chi(1-\Delta^2)}+\sqrt{s^*}\sqrt{\chi}\tilde{\epsilon}]\cdot [\sqrt{\frac{s^*}{s}}\sqrt{\chi(1-\Delta^2)}+\sqrt{s^*}\sqrt{\chi}\tilde{\epsilon}]\nonumber\\
   &= \langle \bar{\beta}^{t+0.5}, \beta^*\rangle-(\sqrt{\frac{s^*}{s}}+\frac{s^*}{s})\chi(1-\Delta^2)-(1+2\sqrt{\frac{s^*}{s}})\sqrt{\chi(1-\Delta^2)}\sqrt{s^*}\sqrt{\chi}\tilde{\epsilon}-(\sqrt{s^*}\sqrt{\chi}\tilde{\epsilon})^2. \label{aeq:32}
\end{align}
For the term $\sqrt{\chi(1-\Delta^2)}$, we have 
\begin{align}\label{aeq:33}
    \sqrt{\chi(1-\Delta^2)}\leq \sqrt{2\chi(1-\Delta)}\leq \sqrt{2\|\bar{\beta}^{t+0.5}\|_2\|\beta^*\|_2-2\langle \bar{\beta}^{t+0.5},\beta^*\rangle }\leq \|\bar{\beta}^{t+0.5}-\beta^*\|_2. 
\end{align}
For the term $\sqrt{\chi}\tilde{\epsilon}$, we have
\begin{align}\label{aeq:34}
    \sqrt{\chi}\tilde{\epsilon}=2\sqrt{\|\bar{\beta}^{t+0.5}\|_2\|\beta^*\|_2}\frac{\|\bar{\beta}^{t+0.5}-\beta^{t+0.5}\|_\infty}{\|\bar{\beta}^{t+0.5}\|_2}\leq \frac{2}{\sqrt{1-k}}\|\bar{\beta}^{t+0.5}-\beta^{t+0.5}\|_\infty.
\end{align}
Plugging (\ref{aeq:33}) and (\ref{aeq:34}) into (\ref{aeq:32}), we get 
\begin{align}
    &\langle \bar{\beta}^{t+1}, \beta^*\rangle \geq \langle \bar{\beta}^{t+0.5}, \beta^*\rangle-(\sqrt{\frac{s^*}{s}}+\frac{s^*}{s})\|\bar{\beta}^{t+0.5}-\beta^*\|_2^2-\nonumber\\
    &(1+2\sqrt{\frac{s^*}{s}})\|\bar{\beta}^{t+0.5}
    -\beta^*\|_2\frac{2\sqrt{s^*}}{\sqrt{1-k}}\|\bar{\beta}^{t+0.5}-\beta^{t+0.5}\|_\infty-\frac{4s^*}{1-k}\|\bar{\beta}^{t+0.5}-\beta^{t+0.5}\|_\infty^2.\label{aeq:35}
\end{align}
Also, since $\|\bar{\beta}^{t+1}\|_2^2+\|\beta^*\|_2^2\leq \|\bar{\beta}^{t+0.5}+\|\beta^*\|_2^2$, subtracting (\ref{aeq:35}), we obtain 
\begin{align}
    \|\bar{\beta}^{t+1}-\beta^*\|_2^2&\leq (1+\sqrt{\frac{s^*}{s}}+\frac{s^*}{s})\|\bar{\beta}^{t+0.5}-\beta^*\|_2^2 +\frac{8s^*}{1-k}\|\bar{\beta}^{t+0.5}-\beta^{t+0.5}\|_\infty^2\nonumber  \\
    &+ (1+2\sqrt{\frac{s^*}{s}})\|\bar{\beta}^{t+0.5}
    -\beta^*\|_2\frac{4\sqrt{s^*}}{\sqrt{1-k}}\|\bar{\beta}^{t+0.5}-\beta^{t+0.5}\|_\infty 
   \nonumber \\
    &\leq (1+2\sqrt{\frac{s^*}{s}}+2\frac{s^*}{s})[\|\bar{\beta}^{t+0.5}-\beta^*\|_2+\frac{2\sqrt{s^*}}{\sqrt{1-k}}\|\bar{\beta}^{t+0.5}-\beta^{t+0.5}\|_\infty]^2 \nonumber \\
    &+\frac{8s^*}{1-k}\|\bar{\beta}^{t+0.5}-\beta^{t+0.5}\|_\infty^2. \label{aeq:36}
\end{align}
Thus, we have 
\begin{equation}
     \|\bar{\beta}^{t+1}-\beta^*\|_2\leq (1+4\sqrt{\frac{s^*}{s}})^{\frac{1}{2}}\|\bar{\beta}^{t+0.5}-\beta^*\|_2+\frac{2\sqrt{2}\sqrt{s^*}}{\sqrt{1-k}}\|\bar{\beta}^{t+0.5}-\beta^{t+0.5}\|_\infty.
\end{equation}
This completes the proof of Lemma \ref{alemma:6}.
\end{proof}

Next, we bound the term $\|\bar{\beta}^{t+0.5}-\beta^*\|_2$ in (\ref{aeq:10}).
\begin{lemma}\label{alemma:7}
Under the assumptions in Theorem 1, the following inequality holds  
\begin{equation}\label{aeq:38}
    \|\bar{\beta}^{t+0.5}-\beta^*\|_2\leq (1-2\frac{\upsilon-\gamma}{\upsilon+\mu})\|\beta^t-\beta^*\|_2. 
\end{equation}
\end{lemma}
\begin{proof}[Proof of Lemma \ref{alemma:7}]
We first note that the self-consistent property in \citep{mclachlan2007algorithm} implies that 
\begin{equation}
    \beta^*=\arg\max_{\beta}Q(\beta; \beta^*),
\end{equation}
which means that $\beta^*$ is a maximizer of $Q(\beta; \beta^*)$. Thus, the proof follows from  the convergence rate of the strongly convex and smooth functions $Q(\beta; \beta^*)$ in \cite{nesterov2013introductory}. For the step size $\eta=\frac{2}{\mu+\upsilon}$, we have 
\begin{equation}
    \|\beta^t+\eta\nabla Q(\beta^t; \beta^*)-\beta^*\|_2\leq (\frac{\mu-\upsilon}{\mu+\upsilon})\|\beta^T-\beta^*\|_2.
\end{equation}
Thus, we get 
\begin{align}
    \|\bar{\beta}^{t+0.5}-\beta^*\|_2&= \|\beta^t+\eta\nabla Q(\beta^t; \beta^t)-\beta^*\|_2\\
    &= \|\beta^t+\eta\nabla Q(\beta^t; \beta^*)-\beta^*\|_2+\eta\|\nabla Q(\beta^t; \beta^*)-\nabla Q(\beta^t; \beta^t)\|_2\\
    &\leq (\frac{\mu-\upsilon}{\mu+\upsilon})\|\beta^T-\beta^*\|_2+\eta \gamma \|\beta^t-\beta^*\|_2.
\end{align}
Taking $\eta=\frac{2}{\mu+\upsilon}$, we complete the proof.
\end{proof}

Combining Lemmas \ref{alemma:6}, \ref{alemma:7}, and equation (\ref{aeq:7}), we have the following lemma.
\begin{lemma}\label{alemma:8}
If
\begin{equation}\label{aeq:44}
    \|\bar{\beta}^{t+0.5}-\beta^*\|_2\leq k \|\beta^*\|_2
\end{equation}
for some $k\in (0,1)$ and further assuming that 
\begin{equation}
    s\geq \frac{4(1+k)^2}{(1-k)^2}s^* \text{ and } \sqrt{s}\alpha \leq \frac{(1-k)^2}{2(1+k)}\|\beta^*\|_2, 
\end{equation}
then it holds with probability at least $1-d^{-3}$ that 
\begin{equation}
    \|\beta^{t+1}-\beta^*\|_2\leq \frac{2}{\upsilon+\mu}\sqrt{s}\alpha+\frac{1}{\upsilon+\mu}\frac{4\sqrt{2}\sqrt{s^*}}{\sqrt{1-k}}\alpha+(1+4\sqrt{\frac{s^*}{s}})^{\frac{1}{2}}(1-2\frac{\upsilon-\gamma}{\upsilon+\mu})\|\beta^t-\beta^*\|_2,
\end{equation}
where $\alpha=C_2\xi(\epsilon\log(nd)+\sqrt{\frac{\log d}{n}})$. 
\end{lemma}
We now prove Theorem 1.
\begin{proof}[Proof of Theorem 1]
By Lemma \ref{alemma:8}, we know that it is sufficient to prove (\ref{aeq:44}), which can be shown by mathematical induction. 

We first prove $\beta^0\in \mathcal{B}$. By assumption, we have $\|\beta^{\text{init}}-\beta^*\|_2\leq \frac{R}{2}$. By the same proof of Lemma \ref{alemma:6}, we can get $\|\beta^0-\beta^*\|_2\leq (1+4\sqrt{\frac{s^*}{s}})^\frac{1}{2}\|\beta^{\text{init}}-\beta^*\|_2\leq (1+4\sqrt{\frac{1}{4}})^\frac{1}{2}\frac{R}{2}\leq R=k\|\beta^*\|_2$. Thus, by Lemma \ref{alemma:7}, we can see that (\ref{aeq:44}) holds for $t=0$. 

Now suppose that (\ref{aeq:44})  holds for all $t\leq k$. Then, we have 
\begin{equation}
        \|\beta^{k+1}-\beta^*\|_2\leq \frac{2}{\upsilon+\mu}\sqrt{s}\alpha+\frac{1}{\upsilon+\mu}\frac{4\sqrt{2}\sqrt{s^*}}{\sqrt{1-k}}\alpha+(1+4\sqrt{\frac{s^*}{s}})^{\frac{1}{2}}(1-2\frac{\upsilon-\gamma}{\upsilon+\mu})\|\beta^k-\beta^*\|_2,
\end{equation}
by assumption we can see that $(1+4\sqrt{\frac{s^*}{s}})^{\frac{1}{2}}(1-2\frac{\upsilon-\gamma}{\upsilon+\mu})\leq \sqrt{1-2\frac{\upsilon-\gamma}{\upsilon+\mu}}$.
Thus, we have 
\begin{equation}
     \|\beta^{k+1}-\beta^*\|_2\leq \frac{1}{\upsilon+\mu}\frac{(2\sqrt{s}+4\sqrt{2}\sqrt{s^*}/\sqrt{1-k})\alpha}{1-\sqrt{1-2\frac{\upsilon-\gamma}{\upsilon+\mu}}}+(\sqrt{1-2\frac{\upsilon-\gamma}{\upsilon+\mu}})^kR.
\end{equation}
By the assumption of $\frac{1}{\upsilon+\mu}\frac{(2\sqrt{s}+4\sqrt{2}\sqrt{s^*}/\sqrt{1-k})\alpha}{1-\sqrt{1-2\frac{\upsilon-\gamma}{\upsilon+\mu}}}\leq ({1-\sqrt{1-2\frac{\upsilon-\gamma}{\upsilon+\mu}}})R$, we have 
\begin{equation}
      \|\beta^{k+1}-\beta^*\|_2\leq ({1-\sqrt{1-2\frac{\upsilon-\gamma}{\upsilon+\mu}}})R+\sqrt{1-2\frac{\upsilon-\gamma}{\upsilon+\mu}}R=R.
\end{equation}
Hence, by Lemma \ref{alemma:7}, we obtain (\ref{aeq:44})  for the case of $t=k+1$. This completes the proof.
\end{proof}

\subsection{Proof of Lemma 2}
From (\ref{eq:5}) it is oblivious that $[\nabla q_i(\beta,\beta))]_j$ is independent of  other $i\in [n]$ for fixed $j\in [d]$. Next, we prove the property of sub-exponential for each coordinate. 

Note that $$[\nabla q_i(\beta,\beta))]_j= [2w_\beta(y_i)-1]y_{i,j}-\beta_j,$$ and $$\mathbb{E}[\nabla q_i(\beta,\beta))]_j= \mathbb{E}(2w_\beta(Y)Y_j-Y_j)-\beta_j.$$ For convenience, we let $\nabla q_{i,j}$  denote $[\nabla q_i(\beta,\beta))]_j$ 
and $\nabla q_{j}$ denote  $\mathbb{E}[\nabla q_i(\beta,\beta))]_j$. 

By the symmetrization lemma in Lemma \ref{alemma:2}, we have the following for any $t>0$
\begin{equation}\label{aeq:50}
    \mathbb{E}\{\exp(t|[\nabla q_{i,j}-\nabla q_j]|)\}\leq  \mathbb{E}\{\exp(t|\epsilon [2w_\beta(y_i)-1]y_{i,j}|)\},
\end{equation}
where $\epsilon$ is a Rademacher random variable. 

Next, we use Lemma \ref{alemma:3} with $f(y_{i,j})=y_{i,j}$, $\mathcal{F}=\{f\}$, $\phi_i(v)=[2w_\beta(y_i)-1]v$ and $\phi(v)=\exp(u\cdot v)$. It is easy to see that $\phi_i$ is 1-Lipschitz. Thus,  by Lemma \ref{alemma:3} we have
\begin{equation}\label{aeq:51}
    \mathbb{E}\{\exp(t|\epsilon [2w_\beta(y_i)-1]y_{i,j}|)\}\leq \mathbb{E}\{\exp[2t|\epsilon y_{i,j}|]\}. 
\end{equation}

By the formulation of the model, we have $y_{i, j}=z_i \beta^*_{j}+v_{i,j}$, where $z_i$ is a Rademacher random variable and $v_{i,j}\sim \mathcal{N}(0, \sigma^2)$. It is easy to see that $y_{i,j}$ is sub-Gaussian and 
\begin{equation}\label{aeq:52}
    \|y_{i,j}\|_{\psi_2}=\|z_i\cdot \beta^*_j+v_{i,j}\|_{\psi_2}\leq C\cdot \sqrt{\|z_i\cdot \beta_j\|^2_{\psi_2}+\|v_{i,j}\|^2_{\psi_2}}\leq C'\sqrt{|\beta_j^*|^2+\sigma^2},
\end{equation}
for some absolute constants $C, C'$, where the last inequality is due to the facts that $\|z_j\beta_j^*\|_{\psi_2}\leq |\beta_j^*|$ and $\|v_{i,j}\|_{\psi_2}\leq C''\sigma^2$ for some $C''>0$. 

Since $|\epsilon y_{i,j}|=|y_{i,j}|$, $\|\epsilon y_{i,j}\|_{\psi_2}=\|y_{i,j}\|_{\psi_2}$ and $\mathbb{E}(\epsilon y_{i,j})=0$, by Lemma 5.5 in \cite{vershynin2010introduction} we have that  for any $u'$ there exists a constant $C^{(4)}>0$ such that 
\begin{equation}\label{aeq:53}
    \mathbb{E}\{\exp(u' \cdot \epsilon \cdot y_{i,j})\}\leq \exp(u'^2\cdot C^{(4)} \cdot (|\beta|_j^2+\sigma^2)). 
\end{equation}
Thus, for any $t>0$ we get
\begin{equation}\label{aeq:54}
      \mathbb{E}\{\exp(2t\cdot |\epsilon \cdot y_{i,j}|)\}\leq 2\exp(t^2\cdot C^{(5)}\cdot (|\beta|_j^2+\sigma^2))
\end{equation}
for some constant $C^{(5)}$. 
Therefore, in total we have the following  for some constant $C^{(6)}>0$
\begin{equation}\label{aeq:55}
        \mathbb{E}\{\exp(t|[\nabla q_{i,j}-\nabla q_j]|)\} \leq \exp(t^2\cdot C^{(6)}\cdot (|\beta|_j^2+\sigma^2))\leq \exp(t^2\cdot C^{(6)}\cdot (|\beta^*|_\infty^2+\sigma^2)).
\end{equation}
Combining this with Lemma \ref{alemma:4} and the definition,  we know that  $\nabla q_{i,j}$ is $O(\sqrt{\|\beta^*\|_\infty^2+\sigma^2})$-sub-exponential. 
\subsection{Proof of Lemma 4}
From (\ref{eq:7}) it is oblivious that $[\nabla q_i(\beta,\beta))]_j$ is independent of other $i\in [n]$ for any fixed $j\in[d]$. Next, we prove the property of sub-exponential. 

Note that $\mathbb{E}\nabla q_{i,j} = \mathbb{E} 2w_\beta(x, y)y\cdot x_j-\beta_j$. Thus, we have 
\begin{equation}\label{aeq:56}
    \nabla q_{i,j}-\nabla q_j= \underbrace{2w_\beta(x_i, y_i)y_i x_{i,j}- \mathbb{E}[]2w_\beta(x,y)y x_j]}_{A}+\underbrace{[x_ix_i^T\beta-\beta]_j}_{B}- \underbrace{y_ix_{i,j}}_{C}.
\end{equation}
For term A and any $t>0$, we have 
\begin{equation}\label{aeq:57}
    \mathbb{E}\{\exp(t|A|)\}\leq \mathbb{E}\{\exp[t|2\epsilon w_\beta(x_i, y_i)y_i x_{i,j}|]\}.
\end{equation}
Using Lemma \ref{alemma:3} on  $f(y_ix_{i,j})=y_ix_{i,j}$, $\mathcal{F}=f$, $\phi_i(v)= 2w_\beta(x,y)v$ and $\phi(v)=\exp(uv)$, we have 
\begin{equation}\label{aeq:58}
     \mathbb{E}\{\exp[t|2\epsilon w_\beta(x_i, y_i)y_i x_{i,j}|]\leq \mathbb{E}\{\exp[4t|\epsilon y_i x_{i,j}|]\}.
\end{equation}
Note that since $y_i=z_i\langle \beta^*,x_i\rangle +v_i$ and  $\|z_i\langle \beta^*,x_i\rangle\|_{\psi_2}=\|\langle \beta^*,x_i\rangle\|_{\psi_2}\leq C\|\beta^*\|_2$ and $\|v_i\|_{\psi_2}\leq C'\sigma$ for some constants $C, C'>0$, by Lemma \ref{alemma8} we know that  there exists a constant $C''>0$ such that 
\begin{equation}\label{aeq:59}
    \|y_i\|_{\psi_2}\leq C''\sqrt{\|\beta^*\|_2^2+\sigma^2}.
\end{equation}
Thus, by Lemma \ref{alemma7} we have 
\begin{equation}\label{aeq:60}
    \|y_ix_{i,j}\|_{\psi_1}\leq \max\{C''^2(\|\beta^*\|_2^2+\sigma^2), C'''\}\leq C_4\max\{\|\beta^*\|_2^2+\sigma^2,1\}. 
\end{equation}
For term B, we have 
\begin{equation}
    \mathbb{E}\{\exp[t |B|]\}=\mathbb{E}\{\exp[t |\sum_{k=1}^d x_{j}x_k\beta_k-\beta_j|]\},
\end{equation}
where $x_j, x_k\sim \mathcal{N}(0, 1)$. Now, by Lemma \ref{alemma7} we have  $\|x_jx_k\beta_k\|_{\psi_1}\leq |\beta_k|C^{(5)}$ for some constant $C^{(5)}>0$. Thus, we get $\|\sum_{k=1}^d x_jx_k\beta_k\|_{\psi_1}\leq C^{(5)}\|\beta\|_1$. 

Also, we know that $\|\beta\|_1\leq \sqrt{s}\|\beta\|_2$, since by assumption $\|\beta\|_0=s$. Furthermore, we have $\|\beta\|_2\leq \|\beta^*\|_2+\|\beta^*-\beta\|_2\leq (1+\frac{1}{32})\|\beta^*\|_2$, since  $\beta\in \mathcal{B}$ (by assumption). From Lemma \ref{alemma6}, we get $\|B\|_{\psi_1}\leq C^{(6)}\sqrt{s}\|\beta^*\|_2$ with some constant $C^{(6)}>0$.

 Thus, we know that there exist some constants $C^{(7)}>0$ and $C^{(8)}>0$ such that 
 
 \begin{align*}
 	\|\nabla q_{i,j}-\nabla q_j\|_{\psi_1} &\leq C^{(7)}\max\{\|\beta^*\|_2^2+\sigma^2, 1\}+ C^{(8)}\sqrt{s}\|\beta^*\|_2 \\
 	&\leq C^{(9)}\max\{\|\beta^*\|_2^2+\sigma^2, 1, \sqrt{s}\|\beta^*\|_2\}.
 \end{align*}
This means that $\nabla q_{i,j}$ is  $O(\max\{\|\beta^*\|_2^2+\sigma^2, 1, \sqrt{s}\|\beta^*\|_2\})$ sub-exponential. 
 
 \subsection{Proof of Lemma 6}
For simplicity, we use notations  $\bar{m}^i=m_\beta(x_i^{\text{obs}}, y_i)$, $\bar{m}=\beta(x^{\text{obs}}, y)$, $\bar{K}^i=K_\beta(x_i^{\text{obs}}, y_i)$, and $\bar{K}=K_\beta(x^{\text{obs}}, y)$. Then, we have 
\begin{equation}
    \nabla q_i -\nabla q=\underbrace{ m_\beta(x_i^{\text{obs}}, y_i)y_i-\mathbb{E} [m_\beta(x_i^{\text{obs}}, y_i)y_i]}_{A}+\overbrace{ \big(K_\beta(x_i^{\text{obs}}, y_i)-\mathbb{E}{K_\beta(x_i^{\text{obs}}, y_i)}\big)\beta}^{B}.
\end{equation}
For the $j$-th coordinate of $A$, we have 
\begin{equation}
    A_j=  \bar{m}^i_j y_i-\mathbb{E} [\bar{m}_j y].
\end{equation}
We note that $\bar{m}_j$ is a zero-mean sub-Gaussian random variable with $\|\bar{m}_j\|_{\psi_2}\leq C(1+kr)$ (see Lemma B.3 in \cite{wang2015high})
\begin{lemma}\label{alemma:11}
Under the assumption of Lemma 6, for each $j\in [d]$, $\bar{m}_j$ is sub-Gaussian with mean zero and $\|\bar{m}_j\|_{\psi_2}\leq C(1+kr)$.
\end{lemma}
Thus, by Lemma \ref{alemma7} we have 
\begin{equation}
    \|\bar{m}_j y_i\|_{\psi_1}\leq C\max\{\|\bar{m}_j\|_{\psi_2}^2, \|y\|_{\psi_2}^2\}\leq C'\max\{(1+kr)^2, \sigma^2+\|\beta^*\|_2^2\},
\end{equation}
where the last inequality is due to the fact that $y=\langle \beta^*, x\rangle+v$. Thus,  $\|y\|_{\psi_2}^2\leq C_3(\|\langle \beta^*, x\rangle\|_{\psi_2}^2+\|v\|_{\psi_2}^2)$ for some $C_3$. 

For term B, we have 
\begin{equation}\label{aeq:65}
    \bar{K}^i_j= \underbrace{(1-z_{i,j})\beta_j}_{C}+\underbrace{\sum_{k=1}^d\bar{m}^i_j\bar{m}^i_k\beta_k}_{D}-\underbrace{\sum_{k=1}^d[(1-z_{i,j})\bar{m}^i_j][(1-z_{i,k})\bar{m}^i_k]\beta_k}_{E}. 
\end{equation}
For term C, we have the following (by Example 5.8 in \cite{vershynin2010introduction})
\begin{equation}
  \|  (1-z_{i,j})\beta_j\|_{\psi_2}\leq |\beta_j|\leq \|\beta\|_{\infty}\leq (1+k)\sqrt{s}\|\beta^*\|_2. 
\end{equation}
For term D, by Lemma \ref{alemma:11} and \ref{alemma7} we have 
\begin{equation}
    \|\sum_{k=1}^d\bar{m}^i_j\bar{m}^i_k\beta_k\|_{\psi_1}\leq \sum_{k=1}^d|\beta_k|\|\bar{m}^i_j\bar{m}^i_k\|_{\psi_1}\leq \sum_{k=1}^d|\beta_k|C^2(1+kr)^2\leq C_4 (1+kr)^2\|\beta\|_1.
\end{equation}
Since $\beta\in \mathcal{B}$, we get $\|\beta\|_1\leq \sqrt{s}\|\beta\|_2\leq (1+k)\sqrt{s}\|\beta^*\|_2$. Thus, we have 
\begin{equation}
     \|\sum_{k=1}^d\bar{m}^i_j\bar{m}^i_k\beta_k\|_{\psi_1}\leq C_4\sqrt{s}(1+kr)^2\|\beta^*\|_2. 
\end{equation}
For term E, since $1-z_i\in [0,1]$, we have $\|(1-z_{i,j})\bar{m}^i_j\|_{\psi_2}\leq \|\bar{m}^i_j\|_{\psi_2}\leq C(1+kr)$. Hence, by Lemma \ref{alemma7} we get  
\begin{align}
   \| \sum_{k=1}^d[(1-z_{i,j})\bar{m}^i_j][(1-z_{i,k})\bar{m}^i_k]\beta_k\|_{\psi_1}&\leq \sum_{k=1}^d|\beta_k| \|[(1-z_{i,j})\bar{m}^i_j][(1-z_{i,k})\bar{m}^i_k]\|_{\psi_1} \nonumber \\
   &\leq \sum_{k=1}^d|\beta_k| C(1+kr)^2\leq C_6(1+kr)^2\sqrt{s}\|\beta^*\|_2.
\end{align}
This gives us  
\begin{equation}
    \|\bar{K}^i_j\|_{\psi_1}\leq C_7\sqrt{s}(1+k)(1+kr)^2\|\beta^*\|_2.
\end{equation}
By Lemma \ref{alemma6}, we get
\begin{equation}
    \|[\nabla q_i-\nabla q]_j\|_{\psi_1}\leq 2\|[\nabla q_i]_j\|_{\psi_1}\leq  C_8[(1+k)(1+kr)^2\sqrt{s}\|\beta^*\|_2+\max\{(1+kr)^2, \sigma^2+\|\beta^*\|_2^2\}]. 
\end{equation}

\end{document}